\documentclass[3p,11pt]{elsarticle}
\usepackage[utf8]{inputenc}
\usepackage{diagbox}
\usepackage{adjustbox}
\usepackage{geometry}
\usepackage{tabularx}
\usepackage{rotating}
\usepackage{multirow}
\RequirePackage{threeparttable}
\RequirePackage{longtable}
\RequirePackage{makecell} 
\usepackage{adjustbox}
\usepackage{amsthm}
\usepackage{thmtools, thm-restate}
\usepackage{hyperref}
\usepackage{amsmath}
\usepackage{bbm}
\DeclareMathOperator*{\argmax}{arg\,max}

\graphicspath{{Figures/}}

\journal{European Journal of Operational Research}

\makeatletter
\def\ps@pprintTitle{%
  \let\@oddhead\@empty
  \let\@evenhead\@empty
  \def\@oddfoot{\parbox{\textwidth}{\footnotesize\itshape
    Accepted manuscript. Published in European Journal of 
    Operational Research, 318(2), 408--423 (2024). 
    \href{https://doi.org/10.1016/j.ejor.2024.05.027}{https://doi.org/10.1016/j.ejor.2024.05.027}}}%
  \let\@evenfoot\@oddfoot
}
\makeatother
\RequirePackage{amsmath,amsthm,amsopn,amstext,amsfonts,amssymb, mathrsfs}
\newtheorem{lemma}{Lemma}

\RequirePackage{algorithm,algorithmicx,algpseudocode}

\RequirePackage{multirow,fancyhdr} 
\RequirePackage{indentfirst}
\RequirePackage{graphics,graphicx,fancybox,color,float,subcaption}
\usepackage{float}
\usepackage{comment}
\usepackage{setspace}
\usepackage{subcaption}

\allowdisplaybreaks

\RequirePackage{multirow}
\RequirePackage{threeparttable}
\RequirePackage{longtable}
\RequirePackage{makecell} 

\RequirePackage{enumerate}
\RequirePackage{soul}

\usepackage[dvipsnames]{xcolor}

\begin{document}

\makeatletter
\def\elsLabel#1{\@bsphack\protected@write\@auxout{}%
  {\string\Newlabel{#1}{\@currentlabel}}\@esphack}
\def\Newlabel#1#2{\expandafter\xdef\csname X@#1\endcsname{#2}}
\def\Ref#1{\@ifundefined{X@#1}{0}{\csname X@#1\endcsname}}
\makeatother

\begin{frontmatter}

\tnotetext[t1]{© 2024. This manuscript version is made available 
under the CC-BY-NC-ND 4.0 license 
\href{https://creativecommons.org/licenses/by-nc-nd/4.0/}{https://creativecommons.org/licenses/by-nc-nd/4.0/}}

\title{Dynamic resource matching in manufacturing
using deep reinforcement learning\tnoteref{t1}}

\author[UH]{Saunak Kumar Panda}
\ead{spanda@uh.edu}
\author[UH]{Yisha Xiang}
\ead{yxiang4@uh.edu}
\author[TTU]{Ruiqi Liu}
\ead{ruiqliu@ttu.edu}

\address[UH]{Department of Industrial Engineering, University of Houston, Houston, TX 77004, USA}
\address[TTU]{Department of Mathematics and Statistics, Texas Tech University, Lubbock, TX 79409, USA}

\begin{abstract}
Matching plays an important role in the logical allocation of resources across a wide range of industries. The benefits of matching have been increasingly recognized in manufacturing industries. In particular, capacity sharing has received much attention recently. In this paper, we consider the problem of dynamically matching demand-capacity types of manufacturing resources. We formulate the multi-period, many-to-many manufacturing resource-matching problem as a sequential decision process. The formulated manufacturing resource-matching problem involves large state and action spaces, and it is not practical to accurately model the joint distribution of various types of demands. To address the curse of dimensionality and the difficulty of explicitly modeling the transition dynamics, we use a model-free deep reinforcement learning approach to find optimal matching policies. Moreover, to tackle the issue of infeasible actions and slow convergence due to initial biased estimates caused by the maximum operator in Q-learning, we introduce two penalties to the traditional Q-learning algorithm: a domain knowledge-based penalty based on a prior policy and an infeasibility penalty that conforms to the demand-supply constraints. We establish theoretical results on the convergence of our domain knowledge-informed Q-learning providing performance guarantee for small-size problems. For large-size problems, we further inject our modified approach into the deep deterministic policy gradient (DDPG) algorithm, which we refer to as domain knowledge-informed DDPG (DKDDPG). In our computational study, including small- and large-scale experiments, DKDDPG consistently outperformed traditional DDPG and other RL algorithms, yielding higher rewards and demonstrating greater efficiency in time and episodes.
\end{abstract}


\end{frontmatter}

\section{Introduction}
Matching plays an important role in the logical allocation of resources across a wide range of industries such as transportation, college admissions \citep{roth1989college}, organ allocation \citep{roth2004kidney}, and online dating. In the transportation sector, matching is the core issue in ride-sharing and its many variants (e.g.,  carpooling, P2P (peer-to-peer) ride-sharing). Ride-sharing has successfully promoted sustainable transportation, reduced
car utilization, increased vehicle occupancy, and public transit among other benefits \citep{mitropoulos2021systematic}. Matching also plays a critical role in organ allocation with the most common example being kidney allocation, where donors and patients are matched based on their compatibility which depends on factors such as organ quality and patient condition.

The benefits of matching have been increasingly recognized in manufacturing industries. In particular, capacity sharing has received much attention recently. Capacity investment is expensive across manufacturing sectors (e.g., semiconductors, and consumer electronics). For example, a new semiconductor fab costs one to four billion dollars to build, and the price for a single machine may be as high as four to five million dollars with a high obsolescence rate \citep{renna2011capacity, wu2005tool}.  In recent years, manufactured products have had a short product life cycle and high demand volatility, making the capacity investment not only expensive but also risky. Capacity sharing provides a viable solution to address the capacity limit facing small manufacturers and helps alleviate the cost burdens large manufacturers carry. Further accelerating the growth of the capacity sharing market is the recent paradigm shift in the manufacturing sector to digital and cloud manufacturing \citep{liu2018cloud} that allows users to request services ranging from product design, manufacturing, testing, management and all other stages of a product life-cycle through the cloud. This provides the critical technical infrastructure for crowd-sourcing and matching between customers and manufacturers. 

A few recent works have studied how to optimize matching in manufacturing. \citet{yang2021matchingam} study a two-sided additive-manufacturing (AM) market for one period and design a bipartite matching framework to match customers with manufacturers. \citet{pahwa2020maas} consider the bipartite matching problem in manufacturing-as-a-service marketplaces in a dynamic environment and propose approximate stable matching algorithms to optimize the revenue for the marketplace platform.  While there have been some recent advances in this field, matching problems in manufacturing industries are still largely under-explored.

To meet the increasing need for resource sharing in manufacturing, we consider a dynamic manufacturing resource matching problem in a finite-time horizon. Although sharing some general characteristics with common matching problems such as heterogeneous supply-demand types and similar reward structures, resource matching in manufacturing is distinctively different in several key aspects. First, the matching of the types needs to be optimized over a finite horizon. That is, at each decision period, there are demands for manufacturing capacities that need to be fulfilled, and this matching process evolves through time. Dynamic matching needs to be differentiated from the instantaneous matching commonly seen in ride-sharing. Second, resource matching typically involves many-to-many matching, since a single order can be fulfilled by multiple manufacturers based on factors such as capacity and distance. Similarly, a single manufacturer can serve multiple customers depending on the type of orders or demand types. Third, the matching framework typically consists of large state and action spaces. State and action space sizes grow exponentially as the number of firms and demand-supply types increases, while many-to-many matching further expands the action space. Lastly, the matching of resources is constrained by manufacturers' capacities. A manufacturer can only share the capacity that is available and a customer is not incentivized to take more than the amount demanded. While feasibility is also an important consideration in an organ allocation problem, a feasibility constraint makes solving a dynamic, many-to-many matching problem significantly more difficult.

In this paper, we explicitly consider the aforementioned characteristics pertaining to resource matching in manufacturing and formulate the problem as a sequential decision process. Specifically, we consider a two-sided matching with random demands and fixed capacities over a finite-time horizon. The matching is many-to-many constrained by the demand and supply quantities. Demands for capacities are allowed to be backlogged. Each matching is associated with a reward and the objective is to maximize the expected total rewards. A challenging modeling element here concerns the transition dynamics of the matching system of interest. It is difficult to accurately model the joint distribution of all types of demands, which inevitably calls for a model-free method. Therefore, we resolve to reinforcement learning (RL) which does not require the knowledge of a probabilistic model for system transitions. Our work represents an initial attempt to solve a complex, dynamic manufacturing resource-matching problem via RL.

Our problem is distinguished by its inclusion of high-dimensional states and actions, with the action space expanding significantly as the state space increases. This challenge is further compounded by demand uncertainty over the planning horizon. Specifically, for a matching problem with $m$ demand types and $n$ supply types, with $N_s$ being the supply limit and $N_d$ being the truncated demand limit, our state space is given by $\mathcal{S} \subseteq \mathbb{R}^{N_d^m} \times m$ with cardinality $|\mathcal{S}|={N_d^m}$, and the corresponding action space is given by $\mathcal{A} \subset \mathbb{R}^{N_s^{n^2}} \times m \times n$ with cardinality $|\mathcal{A}| \in \mathbb{R}^{N_s^{n^2}}$. This issue renders the problem practically intractable by any conventional exact methods or advanced MDP algorithms, especially in applications of our interest such as inventory control\citep{dulacarnold2016deep, Vanvuchelen2022}. The authors illustrate the gravity of the issue using a basic joint replenishment problem, wherein simultaneous replenishment decisions are made for multiple items. With only binary decisions (e.g., order/no order) available for each item, the total number of actions reaches up to a billion in scenarios involving up to 40 items.

While model-free RL eliminates the need for explicit modeling of transition dynamics, conventional tabular model-free methods like Q-learning, while simple and exhibiting excellent learning ability, prove inefficient in solving complex problems in large state-action environments since many of the state-actions might not have been experienced previously \citep{8836506}. Moreover, \citet{jin2018qlearning} have mathematically shown Q-learning to have a runtime of $\mathcal{O}(T)$ and a space of $\mathcal{O}(\mathcal{SA}H)$, where $H$ and $T$ are number of steps per episode and total number of steps respectively. Furthermore, traditional Q-learning employing a greedy target policy often encounters biased estimates early in the learning process, affecting solution quality and convergence speed. To address these issues, we introduce a penalty to the Q-learning update equation to penalize actions deviating from a given prior policy. Additionally, we introduce an infeasibility penalty to tackle encountering infeasible actions. The convergence of our modified Q-learning method with a general penalty function is theoretically guaranteed, with two variations in the update rule based on the behavior of the regularization parameter, $\beta$. We augment the traditional Deep Deterministic Policy Gradient (DDPG) framework by integrating the modified update rule into its architecture, embedding it within the critic loss function to efficiently handle large state and action spaces. This adaptation allows our proposed DKDDPG to explore the solution space adeptly and identify optimal policies efficiently within the dynamic matching problem domain, improving the convergence rate while providing reasonably accurate solutions.

Our approach differs from existing methods in two significant ways. Firstly, the inclusion of a regularizer term, informed by problem-specific prior policies, is expected to enhance Q-learning performance, addressing the issues of biased initial estimates and slow convergence. This regularizer also enables the utilization of high-quality solutions from one-period matching problem solvers, guiding the agent towards optimal policies more effectively and eliminating pathological policies. Secondly, while our modified Q-learning algorithm shows promise in small-scale scenarios, scaling to large-scale matching scenarios necessitates adapting the modified update rule to appropriate deep neural networks. Training traditional deep RL methods such as DQN proves infeasible for many industrial applications that contain large action spaces as shown by \citet{Vanvuchelen2022}. Consequently, we tailor an existing deep RL method, DDPG, by integrating the modified update rule into the critic network's loss function, ensuring performance enhancement.

The main contribution of this paper is three-fold: First, we formulate the multi-period, many-to-many manufacturing resource matching problem as a sequential decision process. This is among the first efforts in providing a multi-period decision framework for resource matching in manufacturing. Second, we prove the convergence of domain knowledge-informed Q-learning algorithm, providing a performance guarantee for small-size problems that can be solved by Q-learning and theoretical guidance for designing efficient algorithms for large-size problems. Lastly, we develop our DKDDPG algorithm which utilizes the domain knowledge-informed Q-learning algorithm, and conduct a computational study to show that it obtains accurate solutions efficiently.

The remainder of the paper is organized as follows. In Section 2, we review relevant literature in the context of matching problems and discuss methods in the RL literature. In Section 3, we design the dynamic resource matching in manufacturing as a sequential decision-making problem using the conventional Markov decision process (MDP) framework. In Section 4, we introduce Q-learning with a penalty and provide theoretical results and proof for its convergence. Section 5 provides details about DDPG and its enhanced version DKDDPG and discusses its implementation for our matching problem. In Section 6, we present computational results comparing the performance of DKDDPG with contemporary RL algorithms in the context of our problem.

\section{Literature Review}
In this section, we review two relevant literature streams: matching and reinforcement learning.

\subsection{Matching Problem}
 
Matching problems can be categorized into three classes: one-to-one, one-to-many, and many-to-many. One-to-one has mostly been studied in the context of ride-sharing \citep{tafreshian2020rideshare}, stable marriage \citep{Knuth1996StableMA}, kidney allocation \citep{roth2004kidney},  car-pooling (e.g., Uber) and home-sharing (e.g., Airbnb). The one-to-one problem is typically modeled as a bipartite matching problem where the bipartite graph $G=(\mathcal{R},\mathcal{D}, E)$ consists of two disjoint sets of nodes: one corresponding to the set of riders ($\mathcal{R}$) and one to the set of drivers ($\mathcal{D}$). Many polynomial-time exact or heuristic methods have been developed for the bipartite matching problem \citep{Riesen2007SpeedingUG}. \citet{tafreshian2020rideshare} review a large variety of exact and approximation algorithms for this problem with polynomial worst-case running time bounds such as greedy algorithm, iterative algorithm based on augmenting path and a scaling algorithm to name a few. \citet{karp1990onlinebm} propose the Ranking algorithm, a simple randomized online algorithm that achieves a competitive ratio. \citet{antoniadis2020secretary} propose using machine learning predictions to solve the online bipartite matching and improve its performance guarantee.

Many-to-one matching problems are those where there are multiple matches possible for a single node. They are commonly seen in college admissions problem designed by \citet{gale1962college}, which is a generalization of the marriage problem. \citet{Baiou2000admissions} characterize the stable admissions polytope using a system of linear inequalities. \citet{sethuraman2006m21} associate a geometric structure to the system of inequalities and provide a simple visual proof of the integrality of the Baïou-Balinski formulation. The many-to-many matching model introduced by \citet{roth1984stability, roth1985job2} is another generalization of the marriage model where multiple nodes on both sides of a bipartite graph can be matched. The matching of demand and supply types is a form of many-to-many matching. Many single-period demand and supply matching problems are modeled as a single-commodity network flow problem (SCNF) and a number of heuristic algorithms have been developed to solve the formulated SCNF problem. For example, to solve the single-commodity, uncapacitated, fixed-charge network flow problem, \citet{ortega2003bnc} develop a branch-and-cut algorithm that determines the right cut set for generating inequalities.

Existing literature on matching problems, while abundant, is mostly defined in a single-period deterministic problem setting. Only a few works have considered dynamic matching in a finite horizon setting. \citet{kurino2014house} and \citet{bloch2009housing2} devise a dynamic allocation problem of on-campus student housing. In their models, agents have deterministic arrivals and departures. \citet{unver2010dynamic} studies dynamic kidney exchange with inter-temporal random arrivals of patient-donor pairs with the objective of maximizing the number of matched compatible pairs. They consider a general dynamic problem from the point of view of a central authority (e.g., a hospital) whose objective is to minimize the long-run total discounted waiting cost. They separately derive efficient dynamic matching mechanisms that conduct two-way and multi-way exchanges, involving the matching of two infeasible donor-patient pairs and multiple infeasible donor-patient pairs respectively. In the context of demand-supply matching, \citet{hu2022dynamic} model dynamic matching of demand-supply types as a finite-horizon problem with a stochastic dynamic program where the optimal expected total discounted surplus is maximized for each demand-supply type. They formulate the problem in an MDP framework, explore the structural properties of the optimal policy, and propose heuristic policies to solve the dynamic matching problem.

Few of the aforementioned works have holistically considered the key characteristics of matching in manufacturing (e.g., many-to-many, dynamic matching, feasibility constraint), and are therefore not applicable to manufacturing resource matching.  On the other hand, limited works that consider matching in manufacturing often make simplified assumptions and are mostly limited to bipartite frameworks and single-period optimization problems. \citet{yang2021matchingam} consider a resource allocation problem between customers and additive manufacturing (AM) manufacturers. They assume that each order demands a single-unit resource and formulate the allocation problem as an integer program. They consider one-to-one matching and the stable matching algorithm is leveraged to optimize matches between customers and AM providers. \citet{pahwa2020maas} propose approximate stable matching solutions using mechanism design and mathematical programming approaches to solve a bipartite matching problem in manufacturing in a dynamic environment. The matching considered in this system is many-to-one between the orders and suppliers. Analytical models and efficient algorithms are a critical necessity for solving large-scale dynamic resource matching in manufacturing.

\subsection{Reinforcement learning}
RL \citep{sutton2018reinforcement} has proven efficient to solve complex sequential decision problems (e.g., telecommunications, robot control, and game playing among others \citep{li2017deep}). RL is concerned with optimizing an agent's choice of action to maximize a cumulative reward. Unlike model-based methods such as dynamic programming, which require an exact transition probability model that is often hard to obtain in practice, model-free RL methods (e.g., Temporal difference (TD(0)) learning methods, Q-learning) aim to learn the optimal policy either online or offline without knowing the underlying transition dynamics.

Despite the fact that Q-learning is guaranteed to find the optimal solution, the runtime increases exponentially as the state and action spaces grow since Q-learning is a tabular method. \citet{mnih2015human} develop an approach to train a deep Q-network (DQN), an action-value function approximated by a convolutional neural network on the high-dimensional visual inputs of a variety of Atari games. As part of their enhancements to the original Q-learning algorithm, two key enhancements are made: experience replay buffer and target network freezing. The former is designed to reduce the instability associated with training on highly correlated sequential data. By using a target network whose weights are periodically derived from the main network, the latter addresses the instability caused by chasing a moving target. Since then, DQN has gained wide popularity for solving problems with high-dimensional sensory inputs and actions and excels at a diverse array of challenging tasks. \citet{al2019deeppool} develop a DQN model to learn optimal dispatch policies for ride-sharing by interacting with the environment. \citet{gao2020application} implement DQN for portfolio management in the stock market and observe that DQN outperforms ten other traditional strategies.

While DQN works well for high-dimensional states with a small action space, it suffers from run-time and memory issues for problems involving very large action space since it increases the size of the output layer. Alternatively, obtaining a single action array as the network output ensures higher efficiency in finding the optimal action. \citet{lillicrap2015continuous} adapt the idea underlying the success of DQN to complex, high-dimensional action spaces and propose an actor-critic model-free algorithm based on the deterministic policy gradient that can operate over continuous action spaces. This algorithm named deep deterministic policy gradient (DDPG) uses the same learning algorithm and network architecture as DQN. It is able to robustly solve many simulated physics tasks involving large state and action spaces.

\subsubsection{RL for Combinatorial optimization}

Numerous advancements have occurred in solution techniques for combinatorial problems sharing similarities. Machine learning (ML) and reinforcement learning (RL) have been widely employed in addressing a variety of common combinatorial problems, including but not limited to the traveling salesman problem (TSP), vehicle routing problem, maximum cut, and minimum vertex cover problem. \citet{bengio2020machine} survey methods in ML and operations research to solve combinatorial problems in general. They review advancements such as imitation learning and experience learning using graph neural networks in the context of combinatorial optimization (CO). \citet{wang2021deep} discuss various recent deep RL methods implemented for common CO problems like capacitated vehicle routing problem (CVRP) and summarize the different deep RL methods categorically based on value and policy, such as DQN, DDPG, and other actor-critic methods. They also compare numerous neural network architectures such as pointer networks, transformers, and LSTMs implemented in conjunction with listed RL algorithms to solve common CO problems.
    
    \citet{delarue2020reinforcement} devise an RL framework tailored for value-based deep RL methods featuring a combinatorial action space. They formulate a CVRP as a sequential decision problem and frame the selection of actions as a mixed-integer optimization problem. \citet{dai2018learning} propose a combination of RL algorithms with graph embedding to learn greedy heuristics for solving widely known CO problems such as the Travelling Salesman Problem (TSP), Maximum cut, and Minimum Vertex cover. \citet{barrett2020exploratory} introduce the ECO-DQN approach, designed to iteratively enhance solutions for the Maximum Cut problem by learning exploration strategies during testing. Their method's adaptability is highlighted by its capability to initiate from any state, demonstrating its flexibility in integration with various search heuristics. \citet{bello2017neural} develop a neural combinatorial optimization algorithm that combines actor-critic methods with recurrent neural networks to solve TSP. They further discuss the possibilities of extending to other problems and emphasize the designing of feasibility for the problems. A similar solution approach was designed by \citet{nazari2018reinforcement}, where they replaced a pointer network and added an attention mechanism to their RNN to expand to VRP as well. \citet{kool2019attention} also implement an attention mechanism along with the REINFORCE methods to learn strong heuristics and solve a wide range of practical problems.
    
    While the aforementioned works provide a thorough literature review of reinforcement learning (RL) applied to combinatorial optimization, our problem stands out due to its distinctive characteristics. With numerous suppliers and customers involved, our problem encompasses multi-dimensional states and actions, with both the action and state spaces expanding exponentially as the number of suppliers and customers grows. This challenge is further compounded by demand uncertainty over the planning horizon. While some cited works tackle infeasibilities, our method adopts a unique approach by proportionally penalizing such solutions and integrating an additional general convex penalty function into the Q-learning update rule, utilizing problem-specific knowledge based on the work by \citet{fox2015taming}. They propose a penalty-based Q-learning algorithm to penalize biased estimates due to the minimum (or maximum) operator in Q-learning. The authors use a prior policy-based penalty function to penalize deterministic policies at the beginning of learning. They show that their method reduces the bias of the value-function estimation, leading to faster convergence to the optimal value and the optimal policy. In this context, we introduce and theoretically validate a regularized Q-learning update rule based on \citet{bertsekas1996neuro} and \citet{singh2000sarsa}. Furthermore, we extend this methodology to address large-scale problems using deep neural networks, demonstrating near-optimal solutions and enhanced overall performance.\\

\section{Problem Setting}
We consider a multi-period manufacturing resource matching problem. Let \(\mathcal{T}={1,2,\ldots,T}\) denote the decision epochs. Let \(\mathcal{D}={1,2,\ldots, m} \) represent the set of demand types for manufacturing resources and \(\mathcal{S}={1,2,\ldots, n}\) represent the set of capacity supply type.  The sets $\mathcal{D}$ and $\mathcal{S}$ are disjoint. The arc $(i,j)$ represents the matching between demand type $i$ ($i \in \mathcal{D}$) and supply type $j$ ($j \in \mathcal{S}$). We denote the complete set of arcs by $\mathcal{A} = \{(i,j)|1\leq i \leq m,1\leq j \leq n\}$. As an example, different types of manufacturing processes such as 3D printing, extrusion, CNC machining, etc. can be considered as types of demand or capacity. Note that capacity and supply are used interchangeably in this paper.

The system state is denoted by the outstanding demand vector $\textbf{x} =(x_{1},\dotso,x_{m})\in \mathbb{R}^{m}_+$. The capacity vector is denoted by $\textbf{c} =(c_{1},\dotso,c_{n})\in \mathbb{R}^{n}_+$, where $x_{i}$ and $c_{j}$ are the quantity of type $i$ demand and type $j$ capacity available to be matched respectively in time $t$. We assume that demands arrive randomly at each decision epoch. We also assume that the capacities are fixed in the planning horizon because the capacity of manufacturing resources typically does not change for some time. At the beginning of period $t$, our state comprises of outstanding (backlogged) demand units of various types, $\textbf{x}$. We assume that each demand will take the capacity for one period. We further assume an exogenous correlation between the distribution of demand between one period and the next, that is, they are determined outside the model and imposed on the model. Our model does not account for endogenous correlations between a matching decision in a period and the demand distribution in a future period.

We consider two widely used reward structures - the horizontal reward model \citep{ashlagi2016optimal} and the vertical reward model \citep{sutton1986vermodel}.  Horizontal matching is preference-based and the reward is dependent on the matching of the demand-supply type pair. The unit matching reward, $r_{ij}$ is based on the distance between the demand type $i$  and capacity type $j$, that is, $\delta_{ij}$. We consider a linear structure reward given by $r_{ij} = R - \delta_{ij}$, where $R$ is the fixed prize for matching. For example, suppose there are two demand types $(i=1,2)$ and two capacity types $(j=1,2)$. Let \(i=j=1\) represent the extrusion process and \(i=j=2\) denote the 3D printing process. Demand type $1$ will form a higher rewarding pair with capacity type $1$ since they have the same preference, and a lower rewarding pair with capacity type $2$. We can extend the same understanding to applications with more than two demand-supply types. Consider manufacturing demand and supply types listed in terms of material (product) upgrading as shown in Figure~\ref{img:product_upgrade}. In this case, matches can be made based on the distance between the demand-supply types, that is, the shorter the distance and thus cost, the higher the reward for matching as we explained previously. On the other hand, vertical matching is quality-based and the reward is dependent on the quality of each demand and capacity type. Vertical matching in manufacturing could be based on aspects such as shop certifications, finishing process, experience, and inspection capabilities which involve differences in terms of quality. For example, the demand-supply types can be lined up from the best demand and supply types to the worst, in terms of quality and have a reward value associated with each type. Vertical matching a pair would generate a cumulative reward based on the reward of the demand type and the supply type as described above.  For simplicity, we consider a linearly additive function. Let  $r_{ij} = f(a_{i},b_{j}) = f_{d}(a_{i}) + f_{s}(b_{j})$, where $a_{i}$ represents the quality of demand type $i$ and $b_{j}$ represents the quality of supply type $j$, such that $f_{d}$ and $f_{s}$ are increasing in $a_{i}$ and $b_{j}$ respectively. 

\begin{figure}[htp]
    \centering
    \includegraphics[width=10cm]{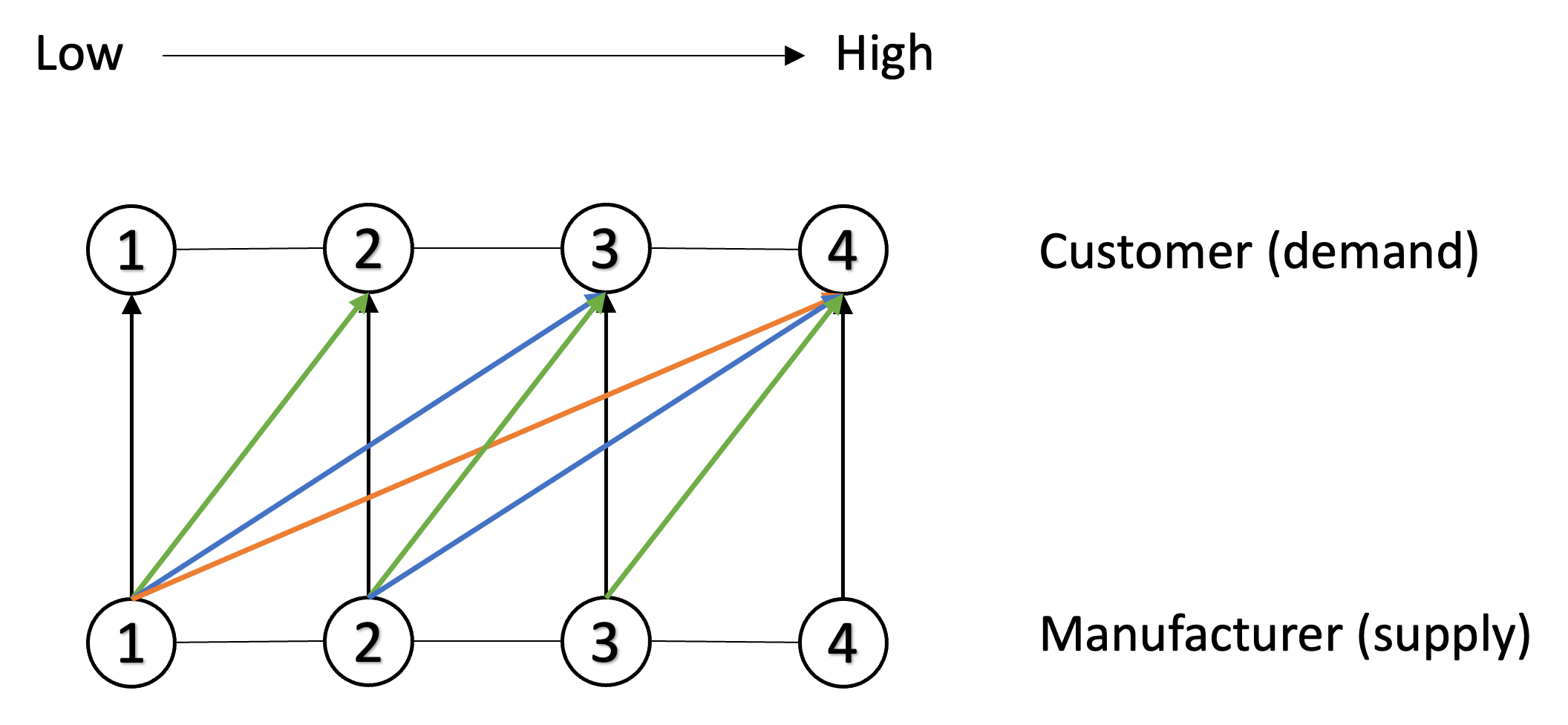}
    \caption{Material (product) upgrade application of Horizontal differentiation}
    \label{img:product_upgrade}
\end{figure}

At each decision period, a decision maker needs to determine the matching quantities for each supply-demand type, $\textbf{Q}= (q_{ij}) \in \mathbb{R}^{m\text{x}n}_+$, where $q_{ij}$ denotes the quantity of the $i^{th}$ demand matched with the $j^{th}$ supply type. The reward associated with each arc $(i,j)$ is modeled based on the horizontal or vertical heterogeneity model. The reward values can be written in a matrix form as $\textbf{R} = (r_{ij}) \in \mathbb{R}^{m\text{x}n}$. The total matching reward is given by $\textbf{R}\circ\textbf{Q} = \sum_{i=1}^m\sum_{j=1}^n r_{ij}q_{ij}$, where ``$\circ$" is the sum of elements of the entry-wise product of two matrices.

\subsection{MDP framework}
In this section, we first formulate our problem as a standard MDP. Let $p_i$ be the probability distribution of demand type \(i\). The total demand of type \(i\) in the next period is the sum of the outstanding demand at the current period \(x_i\) and the demand that occurred in the current period \(d_{i}\). Subtracting the amount of type \(i\) demand filled by matching $\bar{q}_{i} = \sum_{j}q_{ij}$, we can obtain the outstanding demand in the next period denoted by $x'_i$ given by,
\begin{align}
        x'_{i} &= x_{i}+d_{i}-\bar{q}_i.
\end{align}

\begin{figure}[htp]
    \centering
    \includegraphics[width=15cm]{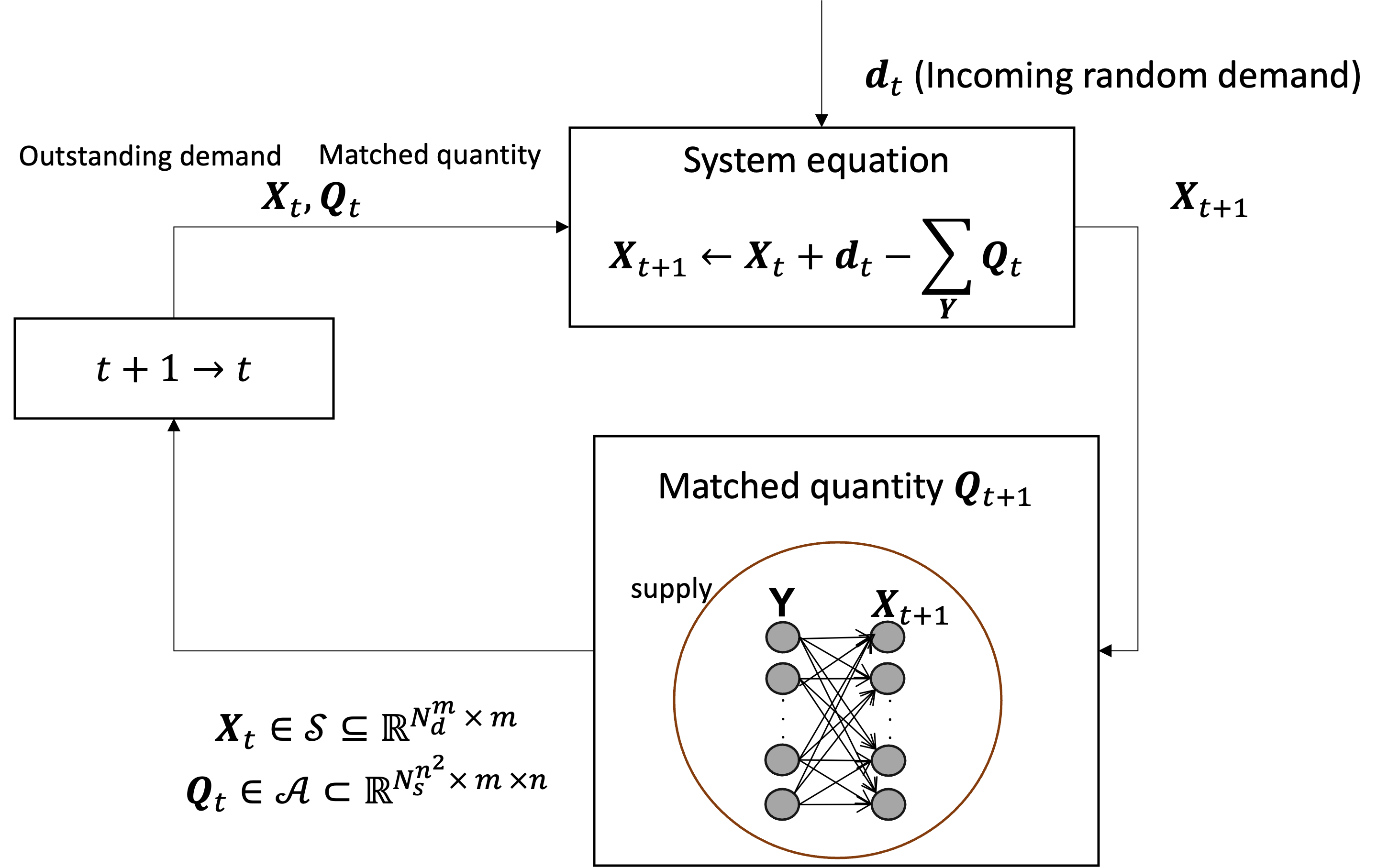}
    \caption{Dynamic matching problem illustration}
    \label{fig:2}
\end{figure}

We have illustrated our dynamic resource matching problem in MDP framework in Figure \ref{fig:2}. Moreover, the matching quantity is constrained by the current state, that is, the demand vector, as well as the capacity vector,
\begin{align}
    \sum_{j=1}^{n}q_{ij} \leq x_{i}, \quad \forall i =1,\dotso, m\\
    \sum_{i=1}^{m}q_{ij} \leq c_{j}, \quad \forall j =1,\dotso, n
\end{align}
The demand type transition probability $p_t$ is given as follows:
\begin{align}
    p_t(x'_{i}|x_{i},\Bar{q}_i) = 
    \begin{cases}
        0         &\text{if } x'_{i} < (x_{i}-\bar{q}_i)\\
        p_{x'_{i}-(x_{i}-\bar{q}_i)} &\text{if }x'_{i}\geq (x_{i}-\bar{q}_i)\\
    \end{cases}
\end{align}

Our goal is to determine a matching policy $\textbf{Q}^* = (q^{*}_{ij})$ that maximizes the expected total discounted reward. Let $\textbf{x}$ be the current state vector in period $t$ and $V_t(\textbf{x})$ be the optimal expected total discounted reward. The Bellman equation that can be used to calculate the total expected discounted reward recursively is given below:
            \begin{align}
                V_{t}(\textbf{x}) = \max_{\textbf{Q}}[\textbf{R}\circ \textbf{Q} + \gamma\sum_{\textbf{x}'}p_{t}(\textbf{x}'|\textbf{x},\textbf{Q})V_{t+1}(\textbf{x}')],
                \label{eq: mdp}
            \end{align}
where $\gamma$ is the discount factor. The matching quantities cannot exceed the demand or supply levels of the different types. At the end of the horizon, all unmatched demand and supply leave the system and therefore, the boundary condition is $V_{T+1}(x)=0$ for all $x \in \mathbb{R}_{+}^{m}$.

The literature on solving this type of matching for a single-period problem is abundant. The above matching problem, however, cannot be solved as a single-period problem which we illustrate with a simple manufacturing example. Consider a 3-period raw-materials matching problem with two demand types and two supply types. Let type 1 be carbon steel and type 2 be stainless steel, that is, a customer who needs carbon steel is demand type 1 and a supplier for the same is supply type 1. It applies similarly to the second demand-supply type. The outstanding demand is $\textbf{x}=[x_1,x_2]=[8, 7]$, and the available capacity is $\textbf{y}=[y_1,y_2] = [6, 5]$. The possible demand quantities for both types are $(0,1,2,3,4,5,6,7,8)$ with probabilities $(0.2,0.2,0.2,0.2,0.2,0,0,0,0)$ for type 1 and $(0.2,0,0.2,0.2,0.2,0,0,0,0.2)$ for type 2 respectively. The discount factor,$\gamma$ is $0.9$ and the reward matrix $\textbf{R}$ is chosen such that $r_{11}>r_{22}>r_{12}>r_{21}$ and is given by \(\textbf{R} = [10 \quad7\quad;\quad5\quad8]\).

We used CPLEX to solve for the single-period solution and used value iteration as the dynamic programming solution. The matching matrices for both solutions are specified below:
\begin{align*}
     \text{Single-period problem:  } &\textbf{Q}_{\textbf{CPLEX}} = [6\quad0\quad ;\quad0\quad5]\\
     \text{Three-period problem:  }&\textbf{Q}_{\textbf{MDP},t_{1}} =[4\quad0\quad; \quad2\quad5], \quad \textbf{Q}_{\textbf{MDP},t_{2}} =[4 \quad0 \quad; \quad2 \quad5],\\ 
    &\textbf{Q}_{\textbf{MDP},t_{3}} =[6 \quad0 \quad; \quad0 \quad5]
\end{align*}
Thus, from the above two solutions, we see that the dynamic programming approach accounts for the incoming demand to obtain an optimal solution as compared to the single-period solution which greedily matches the quantities in every period. Moreover, for an MDP to have a myopic optimum, it requires each transition probability to depend on the action taken but not on the state from which the transition occurs \citep{sobel1981myopic}. Thus, optimal matching policies cannot be obtained by solving a single-period problem.

\subsection{Domain knowledge-informed Q-learning}
The standard MDP problem formulated in Equation~\ref{eq: mdp} requires explicit modeling of the transition dynamics, which can be hard to obtain in practice. Moreover, even if the dynamics are available, standard MDPs suffer from the curse of dimensionality. Conventional value iteration or policy iteration algorithms can thus, solve only small-size problems due to computing limitations. On the other hand, model-free RL algorithms use experiences (e.g., samples) to obtain the value estimate of a state rather than building a model of the environment. The RL algorithms can further leverage powerful function approximation methods to compactly represent value functions, which enables it to deal with large, high-dimensional state and action spaces. In this section, we enhance the existing Q-learning algorithm based on the characteristics of the problem of our interest. Specifically, we first introduce a function that penalizes deviations of the learned policy from some prior policy. We also remove the constraints by penalizing the violation of constraints in the Q value function. The resulting value-penalty function is then implemented in the update step of Q-learning.

 The action-value function (Q-function) is a critical modeling element in any RL algorithm. The Q-value for any state $s \in S$ and $a=\pi(s)$ is given by,
\begin{align}
    Q^\pi(s,a) &= \mathrm{E}[r_t + \gamma r_{t+1} + \gamma^2 r_{t+2} + ...|s_t=s,a_t=a;\pi]\\
           &= r + \gamma\mathrm{E}[V^\pi(s')|s,a], \label{eq:Q-function}
\end{align}
where $s_t, a_t$, and $r_t$ are the state, action, and the reward obtained in time step $t$ respectively, $\pi$ is the policy and $\gamma$ is the discount factor.

Q-learning is a model-free, off-policy RL algorithm, where the Q-value is randomly chosen initially for each state-action pair and is then updated iteratively using the following rule
\begin{align}\label{eq:Q-learning}
    Q(s_t,a_t) \longleftarrow (1-\alpha_t)Q(s_t,a_t) + \alpha_t(r_{t+1} + \gamma\max_{a}Q(s_{t+1},a)), \qquad \forall t=1,2,....,
 \end{align}
where $\alpha_t$ is the learning rate (or step-size).

As observed in Equation~\ref{eq:Q-learning}, the Q-learning algorithm updates a state-action pair by maximizing over all those actions possible in the next state. This can lead to the agent learning biased estimates in the initial stages of the algorithm and thus, much of the initial time would be spent on unlearning the biased estimates. To tackle this issue, inspired by the work by \citet{fox2015taming}, we incorporate a penalty on the learned policy $\pi(a|s)$ for $a\in \mathcal{A}$ and $s\in \mathcal{S}$ with respect to some prior policy $\mu(a|s)$. The penalty function is defined as,
\begin{align}\label{eq:penaltyfn}
    g^\pi(s,a) = h(\pi(a|s),\mu(a|s))
\end{align}
where $h(\cdot,\cdot)$ is any convex function. 

We consider a general convex function as compared to the above-referenced work where the authors specifically use KL-divergence for the penalty function. Other common examples of divergence functions include all types of norms.
The penalty in Equation~\ref{eq:penaltyfn} regularizes the learned policy by penalizing deviations from the prior policy, which is added to the value estimate of the next state. Recall, the value function $V^\pi(s)$ for a given state $s$ with policy $\pi$ is given by,
\begin{align}
    V^\pi(s) = \sum_{t\geq0}\gamma^t\mathrm{E}[r_t|s_0=s]
\end{align}
Adding the penalty function to the discounted reward function in the above equation we get,
\begin{align}
    F^\pi(s) = \sum_{t\geq0}\gamma^t\mathrm{E}\bigg[\frac{1}{\beta}g^\pi(s_t,a_t) + r_t|s_0=s\bigg]
\end{align}
We call the above function $F^\pi$ as the value-penalty function. Here, $\beta$ is the regularization parameter that sets the weight of the penalty. It controls the effect of the regularization in the update step.

The $V^\pi(s)$ in Equation~\ref{eq:Q-function} can be replaced with the value-penalty function $F^\pi(s)$ defined above to give a new Q-function analogous to Equation~\ref{eq:Q-function},
\begin{align}
    Q_{DK}^\pi(s,a) = r + \gamma\mathrm{E}[F^\pi(s')|s,a],
\end{align}
which we call the domain knowledge-informed Q-function. From the above two definitions we obtain
\begin{align}
    F^\pi(s) &= \sum_a \pi(a|s) \bigg[\frac{1}{\beta} g^\pi(s,a) + Q_{DK}^\pi(s,a) \bigg]
\intertext{The domain knowledge-informed Q-learning update equation is then given by,}
    Q_{DK}(s_t,a_t) &\longleftarrow (1-\alpha_t)Q_{DK}(s_t,a_t) + \alpha_t(r_t + \gamma \max_\pi F^\pi(s_{t+1})). \label{eq:newQ}
\end{align}

\subsubsection{Prior policy $\mu$}
Choosing a prior policy is critical since it can either improve convergence or hinder it accordingly. Hence, a prior policy that represents the prior knowledge of the problem has to be carefully chosen.  In our problem, we consider a single-period optimal policy as the prior policy, which is obtained by using a standard solver such as CPLEX. The prior policy helps eliminate pathological policies (e.g., actions that match no quantities) that slow down the convergence.

\subsubsection{Regularization parameter $\beta$}
The regularization parameter, $\beta$ controls divergence from the prior policy. It can either be fixed with respect to time i.e., $\beta_t = \beta$, or scheduled with time, where $\beta_t$ is a function of time. The penalty function $g^\pi(s,a)$ is modeled in a way that when the value of $\beta$ increases to a large number, the optimal domain knowledge-informed Q-learning estimate $Q_{DK}^*$ and value-penalty estimate $F^*$ reduce to standard Q-learning estimate $Q^*$ and value function estimate $V^*$. Also, when $\beta$ is infinitesimal, the effect of the penalty on the value estimate follows the prior policy $\mu$. Thus, in the initial stages of learning, the prior policy gives an advantage over the greedy Q-value estimate, and in the later stages of learning, the greedy Q-value estimate is a more precise estimate of optimal Q-value, $Q^*$. Therefore, scheduling $\beta$ with respect to time would ensure smooth transitioning from $Q^\mu$ to $Q^\pi$, thereby balancing the benefits of both phases of learning.

\subsection{Infeasibility penalty}

An infeasible action is one that violates the demand and capacity constraints established in the earlier section 3.1. This means that the total quantity matched exceeds the respective demand and capacity limits. Hence, it is critical to penalize these actions in order to produce feasible matched quantities. We introduce an infeasibility penalty for violation of each constraint. The penalty associated with violation of the state (outstanding demand) is proportional to the sum of the exceeded matched quantity over all the demand types given by,
\begin{align}
    u(\textbf{x},\textbf{Q}) &= k_1 \bigg(\sum_i^m \mathbbm{1}_{\bar{q}_i > x_i}(\bar{q}_i - x_i)\bigg)
    \intertext{where $u(\cdot,\cdot)$ is the demand penalty function whose arguments are demand $\textbf{x}$ and action (matching quantity) $\textbf{Q}$, $\bar{q}_{i} = \sum_{j=1}^{n}q_{(i,j)}$, and $k_1$ is the demand penalty proportional constant. Similarly, the penalty associated with violation of the capacity is given by,}
    v(\textbf{Q}) &= k_2 \bigg(\sum_j^n\mathbbm{1}_{\bar{q}_j > c_j}(\bar{q}_j - c_j)\bigg)
\end{align}
where $v(\cdot)$ is the supply penalty function whose argument is action (matching quantity) $\textbf{Q}$, $\bar{q}_{j} = \sum_{i=1}^{m}q_{(i,j)}$, and $k_2$ is the supply penalty proportional constant. We obtain a revised reward formulation by subtracting the penalty functions from the matching reward $r_t$ defined in section 3.1,
\begin{align}
    r_t = \textbf{R}\circ\textbf{Q} - u(\textbf{x},\textbf{Q}) - v(\textbf{Q}).
\end{align}

In the following section, we give theoretical results to prove the convergence of domain knowledge-informed Q-learning algorithm for two cases: fixed $\beta$ and changing (increasing) $\beta_t$. Note, that the matching reward $r_t$ in time $t$ encompasses the net reward after including the infeasibility penalties.

\section{Convergence of domain knowledge-informed Q-learning}
By introducing a penalty based on a prior policy, we nudge the algorithm away from learning biased estimates, thereby potentially leading to faster convergence than Q-learning. We formalize this intuition here and prove that the optimal Q-value obtained from the domain knowledge-informed Q-learning algorithm converges to the optimal Q value. We develop our proofs based on theoretical results provided by \citet{bertsekas1996neuro} and \citet{singh2000sarsa}. Note, in both the following sections, we denote domain knowledge-informed Q-function as $Q$ instead of $Q_{DK}$ for ease of notation.

\subsection{Value-penalty function with fixed $\beta$}

In this section, we consider a general algorithm based on pseudo-contraction to help prove the convergence of our domain knowledge-informed Q-learning algorithm with fixed weight parameter $\beta$. 

We first introduce several expressions that are heavily used in the analysis of the convergence behavior of the standard Q-learning algorithm. Starting with some arbitrary estimate $f_0 \in \mathrm{R}^n$, we assume that the $i$th component $f(i)$ of $f$ is updated according to
\begin{align}\label{eq:generaliter}
    f_{t+1}(i) = \big(1-\gamma_t(i)\big)f_t(i) + \gamma_t(i)\big((Hf_t)(i) + \omega_t(i)\big), \qquad t=0,1,\dotso,
\end{align}
where states are denoted by $i=1,2,...,n$, and $\omega_t(i)$ is a random noise term. We denote by $\mathcal{P}_t$ the history of the algorithm until time $t$, which can be defined as,
$\mathcal{P}_t=\{f_{0}(i),...,f_{t}(i),w_{0}(i),...,w_{t}(i),\\\gamma_{0}(i),...,\gamma_{t}(i)\}$, for $i=\{1,2,...,n\}$, or may include some additional information. We now introduce some assumptions to help prove the following theorem.

\begin{restatable}{assumption}{asmpone}(Assumption 4.3, Prop. 4.4, Bertsekas et al.)
\label{asm:ASMONE}
\begin{enumerate}[(a)]
    \item The step-sizes $\gamma_t(i)$ are nonnegative and satisfy
    \begin{align*}
        \sum_{t=0}^\infty\gamma_{t}(i)=\infty, \qquad \sum_{t=0}^\infty\gamma_{t}^2(i)<\infty
    \end{align*}
    \item For every $i$ and $t$, we have $\mathrm{E}[\omega_t(i)|\mathcal{F}_t]=0$.
    \item Given any norm $||\cdot||$ on $\mathrm{R}^n$, there exist constants $A$ and $B$ such that
    \begin{align*}
        \mathrm{E}[\omega_t^2(i)|\mathcal{F}_t] \leq A + B||f_t||^2, \qquad \forall i,t.
    \end{align*}
    \item The mapping $H$ is a weighted maximum norm pseudo-contraction. 
\end{enumerate}
\end{restatable}

Notice that, since $0 \leq \gamma_t(i) < 1$, Assumption~\ref{asm:ASMONE}(a) requires that all state-action pairs be visited infinitely often. Parts (b) and (c) provide assumptions on the noise term. Assumption~\ref{asm:ASMONE}(b) states that $\omega_t(i)$ has zero conditional mean and part (c) provides an upper bound on the conditional variance of the noise term. Part (d) implies that if there exists some $r^* \in \mathbb{R}^n$, a positive vector $\xi = (\xi(1),\dotso,\xi(n)) \in \mathbb{R}^n,$ and a constant $L \in [0,1),$ then the function $H:\mathbb{R}^n \mapsto \mathbb{R}^n$ satisfies,
\begin{align*}
    ||Hr-r^*||_\xi \leq L||r-r^*||_\xi, \qquad \forall r.
\end{align*}

Based on the above assumptions, the proof of convergence of the sequence generated by iteration ~\ref{eq:generaliter} has been given by \citet{bertsekas1996neuro}. We state the convergence result for the Q-function with penalties on deviations from a prior policy and infeasibilities (Equation~\ref{eq:newQ}) under a fixed \(\beta\).
\begin{restatable}{theorem}{thmone}
\label{thm:THMONE}
Let $Q_t$ be the sequence generated by the iteration~\ref{eq:newQ}. We assume that  the step-sizes $\alpha_t$ are non-negative and satisfy
    \begin{align*}
        \sum_{t=0}^\infty\alpha_{t}=\infty, \qquad \sum_{t=0}^\infty\alpha_{t}^2<\infty
    \end{align*}
    Then $Q_t$ converges to $Q^*$ with probability 1.
\end{restatable}

Theorem~\ref{thm:THMONE} shows that Q-learning with penalties on deviation from a priory policy and infeasibilities still converges, in the limit, to the optimal Q-function, under the same mild conditions as traditional Q-learning. We prove this convergence by verifying that the new Q-function in our problem setting also satisfies the conditions provided in Assumption~\ref{asm:ASMONE}. Note that verifying these conditions is not straightforward and requires rigorous analysis of the properties of the new Q-function. Detailed proof for Theorem \ref{thm:THMONE} is provided in Appendix~\ref{sec:appendixA}.

\subsection{Value-penalty function with changing $\beta_t$}
As explained in section 3.2.2, scheduling $\beta$ with respect to time ensures a smooth transition of the learning algorithm from $Q^\mu$ to $Q^\pi$. In this section, we prove the convergence of the domain knowledge-informed Q-learning for an increasing regularization parameter $\beta_t$. Note that the penalty-value function is now dependent on $\beta_t$ instead of $\beta$ and is therefore denoted by $F^\pi_t$. Moreover, the assumptions and results provided previously involve a time-independent mapping $H$ and therefore, cannot be used to prove the convergence for an increasing regularization parameter $\beta_t$. Based on this information, we provide Equation~\ref{eq:newQ} as an iteration,
\begin{align} \label{eq:iternewQ}
    Q_{t+1}(s_t,a_t) = (1-\alpha_t)Q_t(s_t,a_t) + \alpha_t(r_t + \gamma \max_\pi F^\pi_t(s_{t+1}))
\end{align}
We now provide the convergence result for the sequence generated by iteration~\ref{eq:iternewQ}. 
\begin{restatable}{theorem}{thmtwo}
\label{thm:THMTWO}
Let $Q_t$ be the sequence generated by the iteration~\ref{eq:iternewQ}. We assume that  the step-sizes $\alpha_t$ are non-negative and satisfy
    \begin{align*}
        \sum_{t=0}^\infty\alpha_{t}=\infty, \qquad \sum_{t=0}^\infty\alpha_{t}^2<\infty
    \end{align*}
    Then $Q_t$ converges to $Q^*$ with probability 1.
\end{restatable}
Theorem~\ref{thm:THMTWO} shows that Q-learning with a scheduled regularization parameter $\beta_t$ with penalties on deviation from a priory policy and infeasibilities still converges, in the limit, to the optimal Q-function. Detailed proof for Theorem \ref{thm:THMTWO} is provided in Appendix~\ref{sec:appendixB}.

\section{Domain knowledge-informed Deep Deterministic Policy gradient (DKDDPG) Algorithm}
While Q-learning is a model-free approach, it suffers from the curse of dimensionality as the classical approaches like dynamic programming. To address this issue, \citet{mnih2015human} implement DQN to approximate Q-values in the Q-learning algorithm using neural networks. In the case of dynamic manufacturing resource matching, the action space is high dimensional, which means DQN may suffer from memory problems since the output layer is the size of the action space.  

The DDPG algorithm employs an actor-critic network, which is a temporal difference (TD) version of the policy gradient. This approach is inspired by the recent success of DQN to train an RL agent to learn the optimal action to maximize the total reward of matching demand-supply for each state. The algorithm is simple to implement and scale since it only needs an actor-critic network and a learning algorithm such as Q-learning. Based on DDPG \citep{lillicrap2015continuous}, we propose domain knowledge-Informed DDPG (DKDDPG) which utilizes the domain knowledge-informed Q-learning update equation introduced in section 3.2, as opposed to DDPG which employs the traditional Q-learning update equation. Our proposed DKDDPG algorithm is presented in Algorithm~\ref{alg:DKDDPG}. The details about action exploration, DNN approximator, domain knowledge-informed Q-value function, scheduling of $\beta$, action transformation, and the algorithm are discussed below.

\textbf{Action exploration:} In every period, the RL agent matches the demand and supply types, assigning the quantities according to policy $\pi$. For action exploration, we employ the $\epsilon-$greedy policy with exponential $\epsilon-$scheduling. During exploration, we inject a normal random noise into the action obtained from the actor network.

\textbf{DNN:} As discussed before, due to the curse of dimensionality, non-linear functions, and neural network approximators can approximate the Q-values in the Q-learning algorithm. Hence, we employ actor-critic networks based on the success of DQN \citep{mnih2015human}, which solved the issues related to non-stationarity and correlations in the observations by proposing target networks and using experience replay memory. We have implemented two DNNs: an actor network and a critic network. The actor network takes the state as the input and gives an action array as the output. Since the DNN does not provide the actions in a usable format for our problem setting, we perform transformations on the action array which we discuss in detail in the action transformation section. The critic network takes the state and action as input and provides the optimal Q-values as output, and acts as a Q-function approximator. The DNNs are trained using random mini-batches taken from the experience replay iteratively until all episodes are complete. 

We design the architectures of our DNNs similar to the ones implemented by \citet{lillicrap2015continuous}. Considering the state and action spaces, our actor DNN consists of 4 fully connected (FC) layers with shape $[s_{dim},50,200,100, a_{dim}]$ where $s_{dim}$ is the dimension of the state array and $a_{dim}$ is the flattened-array dimension of the action matrix, along with 3 ReLU activation layers in between and the output FC has softmax activation. We use softmax over tanh for the output layer to obtain a soft matching policy for each demand and supply type. The final layer weights of both the actor and critic were initialized from a uniform distribution [-0.003, 0.003]. This is to ensure the initial outputs for the policy and value estimates were near zero. Our critic DNN consists of 3 FC layers with shape $[s_{dim},50,100,200]$, along with 2 ReLU activation layers in between and a linear activation output. We have provided a simple illustration of our neural network architecture in Figure \ref{fig:NN}.

    \begin{figure}[htp]
        \centering
        \includegraphics[width=14cm]{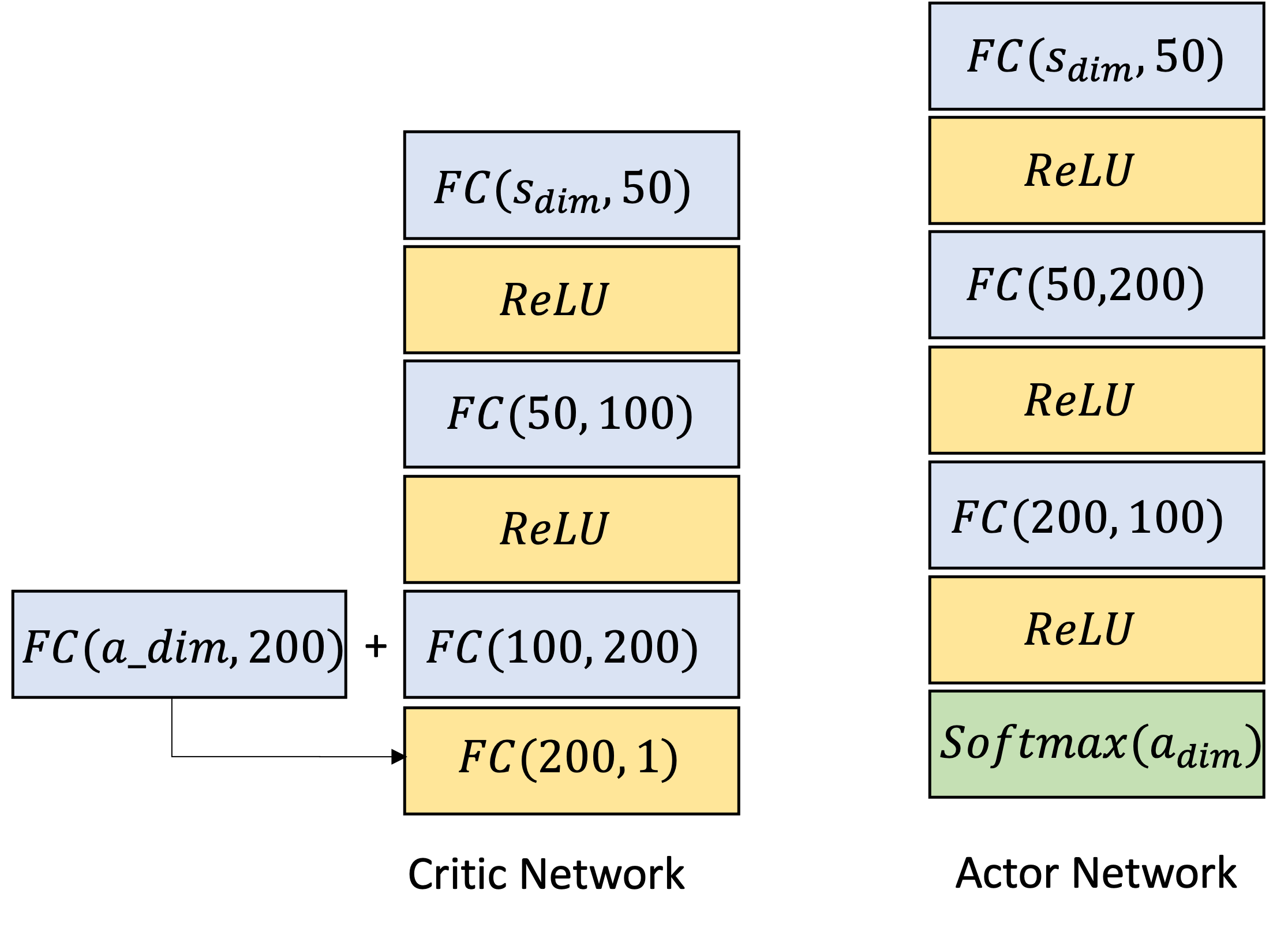}
        \caption{Neural network architecture of DKDDPG}
        \label{fig:NN}
    \end{figure}
    


\textbf{Domain knowledge-informed Q-Value function:} We have incorporated the target value of the domain knowledge-informed Q-learning update equation to step 12 of Algorithm~\ref{alg:DKDDPG}, where the target prediction is given by, 
\begin{align*}
    y_i = r_i + \gamma F^{\pi^*}_i = r_i + \gamma \sum_a \pi^*(a|s) \bigg[\frac{1}{\beta_i} g^{\pi^*}(s, a) + Q'^{\pi^*}(s_{i+1},\mu(s_{i+1}|\theta^{\mu})|\theta^{Q^{'}}) \bigg]
\end{align*} 
where $\pi^* = \argmax_\pi F^\pi_i$ and $Q'$ is the target network with weights $\theta^{Q'}$.  

\textbf{Algorithm:} Algorithm~\ref{alg:DKDDPG} finds weights $\theta^{Q}$ of the critic DNN network to minimize the Euclidean distance between Q-value $Q(s, a;\theta^Q)$ and $y_i$. Our approach uses policy gradient to optimize the weights $\theta^\mu$ of the actor DNN network to maximize the Q-value obtained from the critic. The target networks are updated by having them slowly track the learned network as given in the update equation in Algorithm~\ref{alg:DKDDPG}. This significantly improves the stability of learning of the networks. The action matrices in each training step of the algorithm are obtained by $\epsilon$-greedy policy.

\textbf{Scheduling of $\beta$:} We schedule the relative weight parameter $\beta$ so that the algorithm penalizes the Q-function and prevents it from choosing deterministic policies initially and then over the length of the episodes, it reduces the penalty, thereby reducing the regularization of the Q-function. We use a linear function to schedule $\beta$, i.e. $\beta = kt$, where the hyperparameter $k$ is selected based on a random search from 10 random values followed by a comparison of the total episodic reward over some initial number of episodes.

\begin{algorithm}[htb]
  \caption{Domain knowledge-Informed DDPG}
  \label{alg:DKDDPG}
  \begin{algorithmic}[1]
    \Statex
    \State \textbf{Input: } Number of demand and supply types, number of time periods $T_{max}$, total number of episodes, state space 
    \State Randomly initialize critic network $Q(s,a|\theta^Q)$ and actor $\mu(s|\theta^\mu)$ with weights $\theta^Q$ and $\theta^\mu$
    \State Initialize target network $Q^{'}$ and $\mu^{'}$ with weights $\theta^{Q'} \leftarrow \theta^{Q}$ and $\theta^{\mu'} \leftarrow \theta^{\mu}$
    \State Initialize replay buffer $E$
      \For{episode in total episodes}
      \State Choose a state $s$ arbitrarily from state space
      \For {$t = 1:T_{max}$}
        \State Choose action $a_t$ from $s_t$ using $\epsilon$-greedy policy
        \State Introduce random demand $d$ and take action $a_t$, observe reward $r_t$, next state $s_{t+1}$
        \State Store transition $(s_t,a_t,r_t,s_{t+1})$ in $E$
        \State Sample a random mini-batch of $N$ transitions $(s_i,a_i,r_i,s_{i+1})$ from $E$
        \State Set $y_i = r_i + \gamma \Bigg(\sum_a \pi^*(a|s) \bigg[\frac{1}{\beta_i} g^{\pi^*}(s,a) + Q'^{\pi^*}(s_{i+1},\mu(s_{i+1}|\theta^{\mu})|\theta^{Q^{'}}) \bigg]\Bigg)$
        \State Update critic by minimizing the loss: $L = \frac{1}{N} \sum_{i}(y_i - Q(s_i,a_i|\theta^{Q}))^2$
        \State Update the actor policy using the sampled policy gradient:
        \begin{align*}
            \nabla_{\theta^\mu}J \approx \frac{1}{N}\sum_{i}\nabla_{a}Q(s_i,a_i|\theta^{Q})|_{s=s_{i},a=\mu(s_{i})}\nabla_{\theta^\mu}\mu(s|\theta^{\mu})|_{s_{i}}
        \end{align*}
        \State Update the target networks:
        \begin{align*}
            \theta^{Q^{'}} &\leftarrow \tau \theta^{Q} +(1-\tau)\theta^{Q^{'}}\\
            \theta^{\mu^{'}} &\leftarrow \tau \theta^{\mu} +(1-\tau)\theta^{\mu^{'}}
        \end{align*}
      \EndFor
      \EndFor
  \end{algorithmic}
\end{algorithm}
\textbf{Action transformation and infeasibility penalty:} When performing action exploitation in the $\epsilon-$greedy policy, we use a softmax activation function in the output layer which outputs a vector of probabilities. We then scale the vector accordingly to obtain the quantities matched between the demand and capacity types. The action vector output from the actor network needs to be normalized and then scaled by performing a row-wise multiplication with the state vector $s$, to obtain a scaled action matrix. Since the scaled action matrix could be infeasible in terms of exceeding the total capacity quantity, we employ the infeasibility penalty introduced in section 3.3 into the reward calculation process. This encourages the network to output feasible actions. We demonstrate the above process using a numerical example in our problem setting. Consider a 2x2 matching problem in manufacturing with an outstanding demand vector $x=(12,8)$. Figure~\ref{img:actiontransform} details the action transformation process, where we transform the vector of probabilities to a 2x2 square matrix. This matrix is then normalized over each row signifying the probability of matching each capacity type to each demand type. Finally, we perform a row-wise scalar multiplication of the outstanding demand vector $x$ to the normalized probability matrix and round it to get the desired action matrix displayed in Figure~\ref{img:actiontransform}.

\begin{figure}[htp]
    \centering
    \includegraphics[width=8cm]{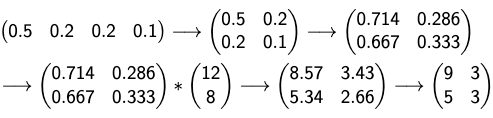}
    \caption{Action transformation for a 2x2 matching example}
    \label{img:actiontransform}
\end{figure}

\section{Computational Study}

In this section, we investigate the performance of domain knowledge-informed DDPG (DKDDPG) with benchmark methods: a solver for LP form of MDP (exact method), Domain Knowledge-informed Q-learning (DKQL), Deep Q-network (DQN), a Deterministic Policy Gradient (DPG) method, and DDPG. We included the tabular version of our modified Q-learning algorithm, DKQL to compare the difference in its performance with its tabular counterpart. Along with the tabular algorithms' comparison, we included two deep RL approaches, namely the DQN and the DPG since the two methods are different in their approach to solving the problem. DQN is a value-based deep RL method that uses a non-linear approximation of the Q-learning update rule as its loss function, and DPG is a policy-based algorithm that learns the near-optimal policy using the principles of policy gradient theorem. To provide an exact solution approach, we have added the linear program formulation and used a solver to obtain the optimal value estimates for small-size problems. Specifically, we first examine the performance of the DKDDPG with that of all the mentioned baselines for small-size problems. For large-size problems that cannot be solved by most of the other methods due to the curse of dimensionality, we compare the performance of the proposed DKDDPG with only the DDPG algorithm.

For all experiments, we consider discrete uniform demands. The lower limits of the demand distributions are set to zero and the upper limits are drawn from a uniform distribution, \(U(0,20)\). The fixed manufacturing capacities are also drawn from this uniform distribution. Since the outstanding demand can go to infinity in theory, we truncate the state space by ignoring states that have little chance to be visited. For simplicity of testing, we assume the number of demand types is the same as that of capacity types (i.e., \(m=n\)). For each problem size, we consider five problem instances that have different demand distributions. In terms of performance metrics, for small-size problems, we consider the average value function estimates (for LP) and average learned Q-value estimates (for all others) over all states, convergence time and the number of episodes to converge. For large-size problems, we compare the convergence time, the number of episodes to converge, and the average learned Q-value estimates over all states for both DKDDPG and DDPG. Although conventionally, the average episodic reward is reported for comparison of algorithm performance, the learned Q-value estimates showcase more stable results and consistent trends as compared to the undiscounted cumulative reward values obtained in an episode. Therefore, we report the learned Q-value estimates in our experiments.

For the model-free RL methods, we train the agent for 3000 episodes and truncate every episode after 500 timesteps. We perform a random search for hyperparameter tuning for DDPG and DKDDPG and obtain the following values for the hyperparameters. The replay memory $E$ is equal to one million most recent experiences, the batch size is 64, the actor learning rate is 0.0001, the critic learning rate is 0.0005, and the soft target update parameter, $\tau$ is 0.0005. We follow an $\epsilon-$greedy policy for action exploration, where $\epsilon$ is annealed exponentially from 1 to 0.1 over the first 300 episodes for small size cases and over the first 1000 episodes for the large size experiments and is fixed thereafter respectively. The computing infrastructure used is an Intel Xeon with 4 cores and 8 logical processors. We consider the maximum run-time to be 150000 seconds for all the algorithms.

\subsection{Performance of DKDDPG in Small size problems}
We compare the performance of DKDDPG with benchmark algorithms including value iteration, Q-learning, and standard DDPG. We test the performance of the algorithms for $m$=2,3,4 and 5. We set the Q-value stopping criterion to be 2\%. Along with the metrics mentioned previously, we also record the percentage difference of the maximum Q-value obtained by Q-learning, DDPG, and DKDDPG, with respect to value estimates obtained through Value Iteration.

\begin{sidewaystable}
\begin{adjustbox}{width=\textwidth,center}
\begin{tabular}{|c|c|cc|cc|cc|cc|cc|cc|cc|}
\hline
\multirow{2}{*}{\begin{tabular}[c]{@{}c@{}}Case\\ (m x n)\end{tabular}} & \multirow{2}{*}{Instance} & \multicolumn{2}{c|}{Linear Programming} & \multicolumn{2}{c|}{Q-learning} & \multicolumn{2}{c|}{DK Q-learning} & \multicolumn{2}{c|}{DPG} & \multicolumn{2}{c|}{DQN} & \multicolumn{2}{c|}{DDPG} & \multicolumn{2}{c|}{DKDDPG} \\ 
\cline{3-16}
 &  & \begin{tabular}[c]{@{}c@{}}Convergence\\ time (s)\end{tabular} & \begin{tabular}[c]{@{}c@{}}Average \\ value\\ estimates\end{tabular} & \begin{tabular}[c]{@{}c@{}}Convergence\\ time (s)\\ /\\ episodes\end{tabular} & \begin{tabular}[c]{@{}c@{}}Average\\Q-values\\/\\\% \\ difference\end{tabular} & \begin{tabular}[c]{@{}c@{}}Convergence\\ time (s)\\ /\\ episodes\end{tabular} & \begin{tabular}[c]{@{}c@{}}Average\\Q-values\\/\\\% \\ difference\end{tabular} & \begin{tabular}[c]{@{}c@{}}Convergence\\ time (s)\\ /\\ episodes\end{tabular} & \begin{tabular}[c]{@{}c@{}}Average\\Q-values\\/\\\% \\ difference\end{tabular} & \begin{tabular}[c]{@{}c@{}}Convergence\\ time (s)\\ /\\ episodes\end{tabular} & \begin{tabular}[c]{@{}c@{}}Average\\Q-values\\/\\\% \\ difference\end{tabular} & \begin{tabular}[c]{@{}c@{}}Convergence\\ time (s)\\ /\\ episodes\end{tabular} & \begin{tabular}[c]{@{}c@{}}Average\\Q-values\\/\\\% \\ difference\end{tabular} & \begin{tabular}[c]{@{}c@{}}Convergence\\ time (s)\\ /\\ episodes\end{tabular} & \begin{tabular}[c]{@{}c@{}}Average\\Q-values\\/\\\%\\difference\end{tabular} \\ 
\hline
\multirow{5}{*}{2 x 2} & 1 & 27662 & 987.4 & 1304/2986 & 917.5/7.1 & 1333/2767 & 924.4/6.4 & 1684/1266 & 971.5/1.6 & 84431/2526 & 969.4/1.8 & 1195/145 & 948.2/3.9 & 1568/138 & 959.0/2.8 \\
 & 2 & 12988 & 700.3 & 2046/2888 & 659.8/5.8 & 2057/2794 & 663.7/5.2 & 1055/802 & 450.1/35.7 & 39287/2355 & 695.4/0.7 & 1279/146 & 687.1/1.9 & 1725/145 & 679.0/3.0 \\
 & 3 & 22956 & 798.6 & 1189/2967 & 746.8/6.5 & 1184/2737 & 751.2/5.9 & 1569/1189 & 699.8/12.4 & 50311/2968 & 796.9/0.2 & 1178/148 & 763.3/4.4 & 1513/133 & 749.8/6.1 \\
 & 4 & 23478 & 890.4 & 1281/2742 & 830.1/6.7 & 1268/2689 & 832.2/6.5 & 2063/1546 & 583.7/34.4 & 93936/2816 & 896.8/0.7 & 1089/142 & 865.9/2.7 & 1555/142 & 863.7/2.9 \\
 & 5 & 19650 & 701.3 & 2040/2966 & 663.1/5.4 & 2046/2876 & 664.7/5.2 & 942/718 & 556.4/20.6 & 26437/2669 & 700.2/0.1 & 1172/155 & 673.2/4.0 & 1568/146 & 680.9/2.9 \\ 
\hline
\multirow{5}{*}{3 x 3} & 1 & 27403 & 504.3 &23413/2925  &453.7/10.1  &24362/2904  &459.2/8.9  & 2552/1918 & 353.7/29.9 &132051/2887  &498.3/1.2  & 1160/150 & 478.2/5.2 & 1590/145 & 480.3/4.7 \\
 & 2 & 10784 & 492.5 &12792/2796  &450.7/8.5  &12932/2642  & 454.6/7.7 & 2151/961 & 481.0/2.3 & 93833/2031 & 473.8/3.8 & 1252/158 & 476.9/3.2 & 1723/154 & 464.8/5.6 \\
 & 3 & 16883 & 502.3 &22119/2874  &451.4/10.1  &23216/2824  &455.3/9.3  & 1284/956 & 349.5/30.4 & 127107/2720 & 500.9/0.3 & 1215/156 & 474.9/5.4 & 1621/147 & 469.5/6.5 \\
 & 4 & 55706 & 651.2 &26229/3000  &596.5/8.4 &26732/3000  &611.2/6.2  & 790/587 & 526.3/19.2 &137214/3000  &646.1/0.8  & 1896/242 & 616.0/5.4 & 1725/154 & 600.6/7.8 \\
 & 5 & 13462 & 491.6 &12053/2986  &449.3/8.6  & 12765/2733 &448.2/8.8  & 3203/2426 & 483.0/1.7 & 94707/2049 & 475.3/3.3 & 1678/152 & 480.4/2.3 & 2293/291 & 474.0/3.6 \\ 
\hline
\multirow{5}{*}{4 x 4} & 1 & \multicolumn{2}{c|}{\multirow{5}{*}{NA}} & \multicolumn{2}{c|}{\multirow{5}{*}{NA}} & \multicolumn{2}{c|}{\multirow{5}{*}{NA}} &4219/2726  &657.2/-  & \multicolumn{2}{c|}{\multirow{5}{*}{NA}} & 2692/333 & 1475.1/- & 2629/233 & 1506.0/- \\
 & 2 & \multicolumn{2}{c|}{} & \multicolumn{2}{c|}{} & \multicolumn{2}{c|}{} &2037/1388  &397.1/-  & \multicolumn{2}{c|}{} & 3016/378 & 767.7/- & 2459/219 & 856.2/- \\
 & 3 & \multicolumn{2}{c|}{} & \multicolumn{2}{c|}{} & \multicolumn{2}{c|}{} &3964/2647  &606.7/-  & \multicolumn{2}{c|}{} & 1683/214 & 1249.0/- & 1689/151 & 1275.3/- \\
 & 4 & \multicolumn{2}{c|}{} & \multicolumn{2}{c|}{} & \multicolumn{2}{c|}{} &1698/1186  &260.8/-  & \multicolumn{2}{c|}{} & 2787/315 & 772.4/- & 2605/227 & 1085.8/- \\
 & 5 & \multicolumn{2}{c|}{} & \multicolumn{2}{c|}{} & \multicolumn{2}{c|}{} &3245/2138  &171.5/-  & \multicolumn{2}{c|}{} & 2386/273 & 835.1/- & 2486/202 & 839.5/- \\ 
\hline
\multirow{5}{*}{5x5} & 1 & \multicolumn{2}{c|}{\multirow{5}{*}{NA}} & \multicolumn{2}{c|}{\multirow{5}{*}{NA}} & \multicolumn{2}{c|}{\multirow{5}{*}{NA}} &2796/1964  &767.8/-  & \multicolumn{2}{c|}{\multirow{5}{*}{NA}} & 2269/259 & 1941.2/- & 3588/301 & 2048.7/- \\
 & 2 & \multicolumn{2}{c|}{} & \multicolumn{2}{c|}{} & \multicolumn{2}{c|}{} &2254/978  &844.5/-  & \multicolumn{2}{c|}{} & 8784/1035 & 1105.4/- & 3533/311 & 1133.0/- \\
 & 3 & \multicolumn{2}{c|}{} & \multicolumn{2}{c|}{} & \multicolumn{2}{c|}{} &2526/1565  &422.4/-  & \multicolumn{2}{c|}{} & 5739/708 & 393.3/- & 3311/298 & 1633.2/- \\
 & 4 & \multicolumn{2}{c|}{} & \multicolumn{2}{c|}{} & \multicolumn{2}{c|}{} &2667/1822  &705.9/-  & \multicolumn{2}{c|}{} & 4753/600 & 1362.2/- & 4545/403 & 1384.5/- \\
 & 5 & \multicolumn{2}{c|}{} & \multicolumn{2}{c|}{} & \multicolumn{2}{c|}{} &977/662  &806.8/-  & \multicolumn{2}{c|}{} & 2914/366 & 932.6/- & 2292/205 & 1056.7/- \\
\hline
\end{tabular}
\end{adjustbox}
\caption{Algorithm Performance in Matching Problem for small size experiments}
\label{tab:smallsize}
\end{sidewaystable}

    

\begin{figure}[!t]
\begin{minipage}[h]{1.0\linewidth}
\centering
\includegraphics[width=\linewidth]{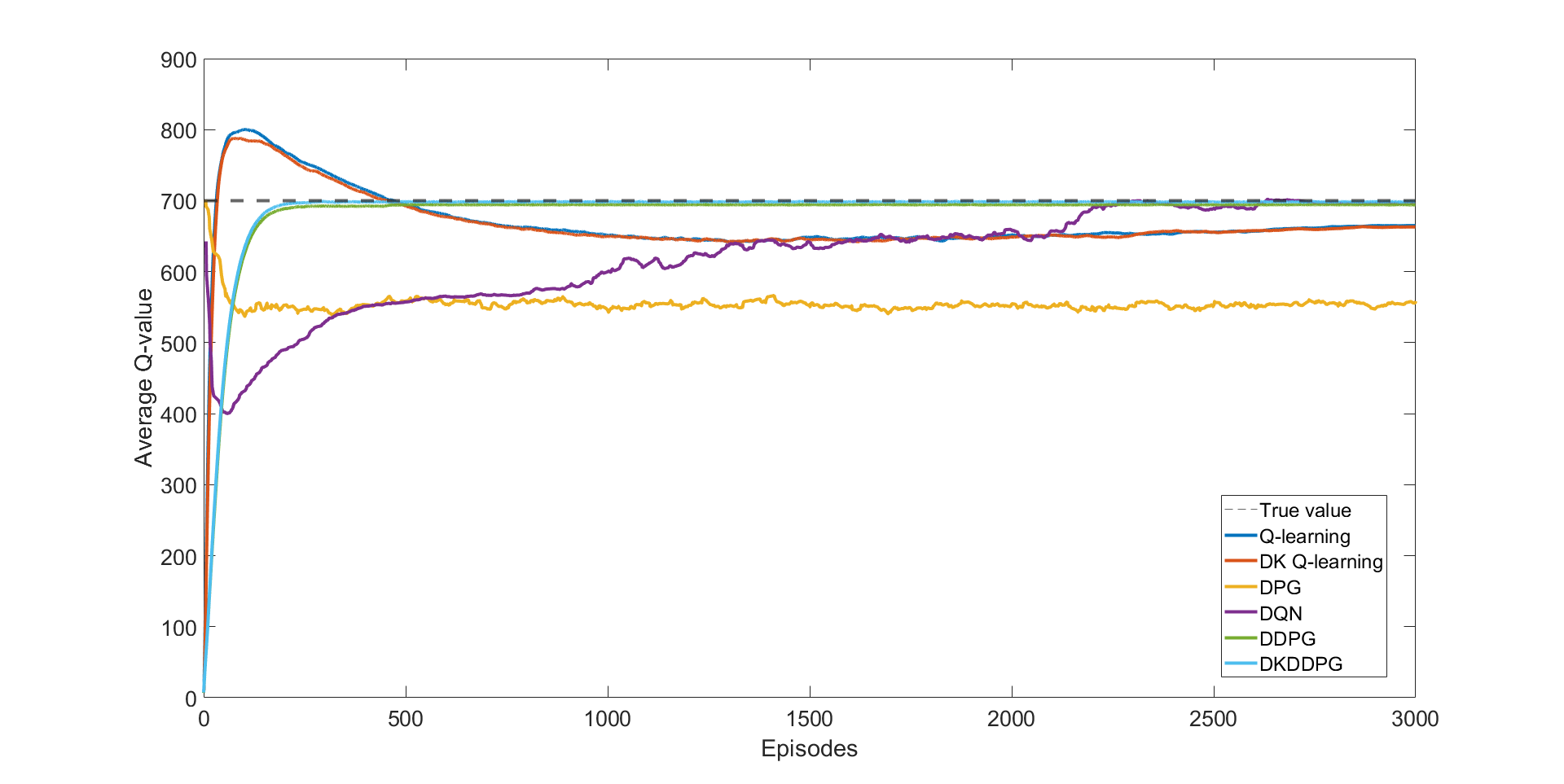}
\subcaption{2x2 Matching problem}
\end{minipage}

\vspace{0.00mm} 

\begin{minipage}[h]{1.0\linewidth}
\centering
\includegraphics[width=\linewidth]{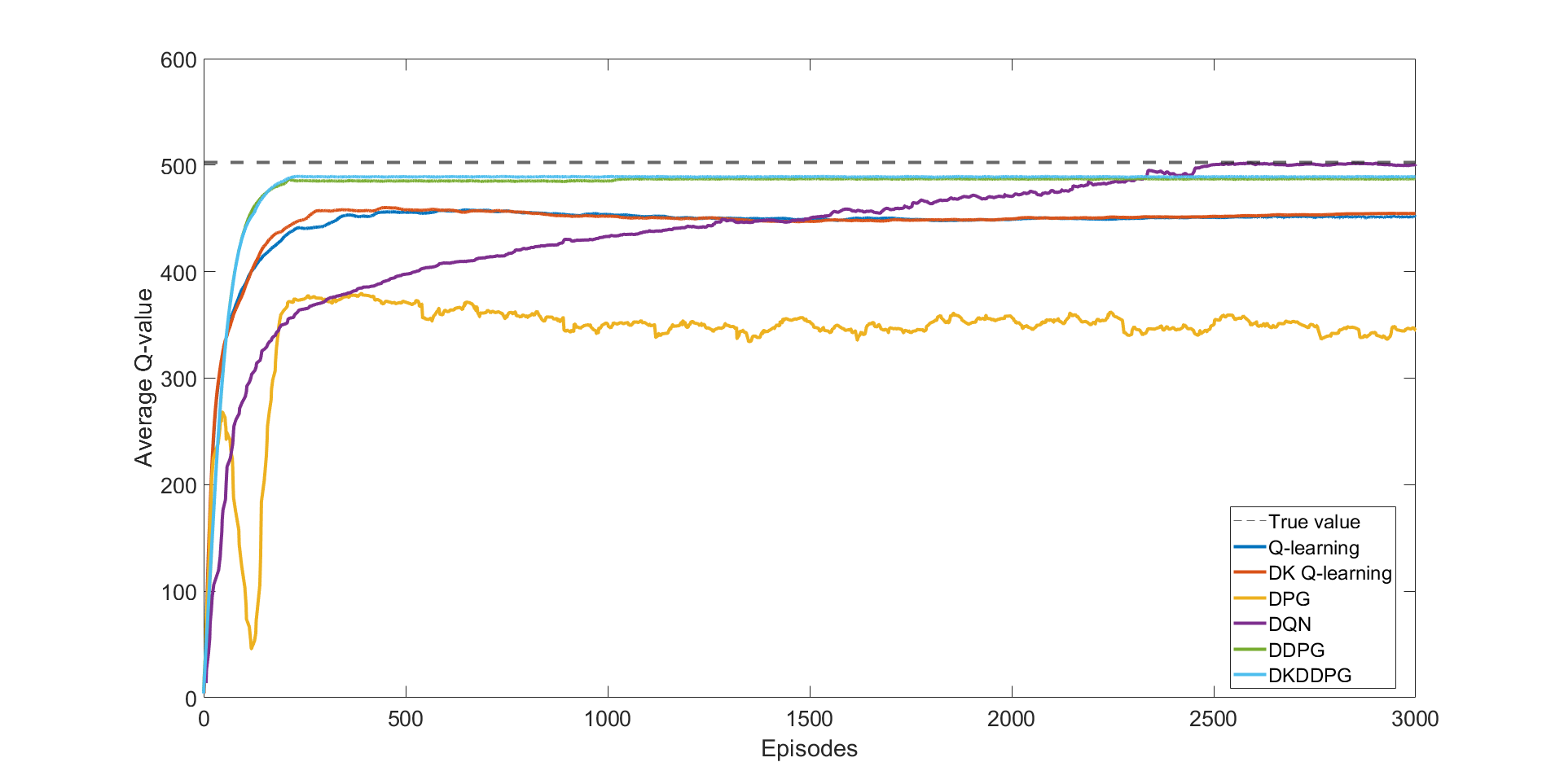}
\subcaption{3x3 Matching problem}
\end{minipage}   


\caption{Comparison of different RL approaches for small-size experiments}
\label{fig:small-case}
\end{figure}


From Table~\ref{tab:smallsize}, we observe that actor-critic methods such as DPG, DDPG and DKDDPG converge in a much shorter time for all the cases as compared to tabular methods such as Q-learning and DK Q-learning, as well as value-based DRL such as DQN. As the size of the cases increases, the time taken to converge by exact methods such as LP and tabular methods such as Q-learning and DK Q-learning increases significantly since they suffer from the curse of dimensionality. For the 4x4 and 5x5 cases, both LP and the tabular RL methods suffer from memory issues and thus, fail to solve the problem in the designated maximum run times or episodes. Moreover, DQN while able to converge quite close to the true Q-values, suffers from similar issues for 4x4 and 5x5 cases due to the large action spaces. On the contrary, DDPG is able to obtain the solutions for all the cases in an average of 2507 seconds and 304 episodes, while DKDDPG is able to solve the matching problem for all the cases taking an average of 2300 seconds and 207 episodes. While DPG also converges at similar times to DDPG and DKDDPG, it obtains much worse values due to correlation and stability issues. Thus, DKDDPG converges approximately 8\% faster than DDPG and takes much fewer episodes to converge. Moreover, Q-values generated by the DKDDPG algorithms converge very close to the values of the LPs as can be seen from the error percentages, implying that by leveraging domain knowledge, the DKDDPG algorithm leads to satisfactory policies while accelerating the converging process. We have also illustrated the training process for two of the small-size experiments in Figure \ref{fig:small-case} to highlight DKDDPG's superior performance for our dynamic matching problem. Among all the methods, only DQN, DDPG, and DKDDPG closely approximate the true value. While DQN generally converges closest to the true value, it typically takes around 2500 episodes to reach near-optimal results. In contrast, DDPG and DKDDPG achieve reasonably accurate solutions in an average of 200-300 episodes, with DKDDPG slightly outpacing DDPG in convergence and yielding slightly better results for smaller cases.

\subsection{Performance of DKDDPG in High-dimensional problems}
Next, we compare the performance of the DKDDPG algorithm with the DDPG algorithm for high dimensional state spaces. Specifically, we consider $m$=10, 15, 20, 25, and 30. We break this section into two parts: First, we perform the training and report the in-sample performance of DKDDPG and DDPG. Second, we check out-of-sample performance by evaluating the trained models of both algorithms on the manufacturing resource matching problem and report the cumulative reward of both models over 500 timesteps.

\subsubsection{In-sample performance of DKDDPG}
For the training phase, we set the Q-value stopping criterion to be 3\%. We chose the L-2 norm as the penalty function $g^\pi(\cdot,\cdot)$. We use the single-period optimal policy discussed in section 3.2.1 as the prior policy $\mu(a|s)$. Since we consider a linearly scheduled $\beta_t$ and run the program for ten random values of $\kappa$, we record the performance metrics for DKDDPG with the $\kappa$ value that obtains the highest average learned Q-values after convergence or program completion. Along with the metrics mentioned before, we also calculate the percentage change for the Q-values, convergence time, and the number of episodes to converge for DKDDPG over DDPG.

\begin{sidewaystable}
\centering
\begin{adjustbox}{width=\textwidth,center}
\begin{tabular}{|c|c|ccc|cccccc|}
\hline
\multirow{2}{*}{\begin{tabular}[c]{@{}c@{}}\\Case\\ (mxn)\end{tabular}} & \multirow{2}{*}{\begin{tabular}[c]{@{}c@{}}\\Instance\end{tabular}} & \multicolumn{3}{c|}{DDPG}                                                                                                                                        & \multicolumn{6}{c|}{DKDDPG}                                                                                                                                                                                                                                                                                                                                                                                  \\ \cline{3-11} 
                                                                      &                           & \begin{tabular}[c]{@{}c@{}}Average \\ Q-values\end{tabular} & \begin{tabular}[c]{@{}c@{}}Time-to-\\convergence\\ (s)\end{tabular} & \begin{tabular}[c]{@{}c@{}}Episodes-to-\\convergence\end{tabular} & \begin{tabular}[c]{@{}c@{}}Average \\ Q-values\end{tabular} & \begin{tabular}[c]{@{}c@{}}Reward\\ Increment/Decrement\\ (\%)\end{tabular} & \begin{tabular}[c]{@{}c@{}}Time-to-\\convergence\\ (s)\end{tabular} & \begin{tabular}[c]{@{}c@{}}Time\\ Increment/Decrement\\ (\%)\end{tabular} & \begin{tabular}[c]{@{}c@{}}Episodes-to-\\convergence\end{tabular}   & \begin{tabular}[c]{@{}c@{}}Episodes\\ Increment/Decrement\\ (\%)\end{tabular} \\ \hline
\multirow{6}{*}{10x10}                                                & 1                         & 556.2                                                             & 25132                                                             & 2871                    & 2674.3                                                             & \textbf{380.8}                                                                        & 18512                                                             & \textbf{-26.3}                                                                     & \multicolumn{1}{c}{1858} & \textbf{-35.3}                                                                         \\
                                                                      & 2                         & 510.8                                                              & 12311                                                             & 1201                    & 2506.7                                                             & \textbf{390.7}                                                                       & 14751                                                             & 19.8                                                                      & \multicolumn{1}{c}{1099} & \textbf{-8.5}                                                                          \\
                                                                      & 3                         & 411.5                                                              & 29712                                                             & 2156                    & 3072.2                                                             & \textbf{646.6}                                                                       & 14414                                                             & \textbf{-51.5}                                                                     & \multicolumn{1}{c}{1057} & \textbf{-50.9}                                                                         \\
                                         &  4                        &  681.3                                                            & 12142                                                             &  1147                   &  3184.5                                                            & \textbf{367.4}                                                                        &  11556                                                            &  \textbf{-4.8}                                                                     & \multicolumn{1}{c}{841} & \textbf{-26.7}                 \\ 
                                         &  5                        &  582.3                                                            & 26055                                                             & 1983                    &  3521.4                                                            & \textbf{504.7}                                                                        & 26679                                                             & 2.4                                                                     & \multicolumn{1}{c}{1668} & \textbf{-15.9}                 \\\hline
                                         & Average                         &  548.4                                                            & 21070                                                             & 1872                    &  2991.8                                                            & \textbf{458.0}                                                                        & 17182                                                             & \textbf{-12.1}                                                                     &  \multicolumn{1}{c}{1305} & \textbf{-27.5}                \\\hline
\multirow{5}{*}{15x15}                                                & 1                         & 886.6                                                             & 53512                                                             & 2227                    & 5189.2                                                             & \textbf{485.3}                                                                        & 52043                                                             & \textbf{-2.7}                                                                      & \multicolumn{1}{c}{1848} & \textbf{-17.0}                                                                         \\
                                                                      & 2                         & 823.9                                                              & 45079                                                             & 2513                    & 6233.4                                                             & \textbf{656.6}                                                                       & 35171                                                             & \textbf{-21.9}                                                                     & \multicolumn{1}{c}{1678} & \textbf{-33.2}                                                                         \\
                                                                      & 3                         & 782.3                                                             & 16687                                                             & 1761                    & 6424.5                                                             & \textbf{721.2}                                                                        & 4367                                                              & \textbf{-73.8}                                                                     & \multicolumn{1}{c}{312}  & \textbf{-82.3}                                                                         \\                              &  4                        &  4765.2                                                            &  32375                                                            & 1931                    &  6625.1                                                            &  \textbf{39.0}                                                                       &  31229                                                            &  \textbf{-3.5}                                                                    & \multicolumn{1}{c}{1575} & \textbf{-18.4}                  \\ 
                                         &  5                        &  976.5                                                            &  36994                                                            &   1771                  &    5243.3                                                          &   \textbf{436.9}                                                                      &   37934                                                           &    2.5                                                                  & \multicolumn{1}{c}{1605} & \textbf{-9.4}                  \\\hline
                                         & Average                         &  1646.9                                                            & 36929                                                             & 2041                    &  5943.1                                                            & \textbf{467.8}                                                                        & 32149                                                             & \textbf{-19.8}                                                                     &  \multicolumn{1}{c}{1404} & \textbf{-32.1}                 \\\hline
\multirow{5}{*}{20x20}                                                & 1                         & 22996.4                                                             & 8862                                                              & 598                     & 55113.2                                                             & \textbf{139.6}                                                                        & 6026                                                              & \textbf{-32.0}                                                                     & \multicolumn{1}{c}{329}  & \textbf{-44.9}                                                                         \\
                                                                      & 2                         & 9702.2                                                             & 10958                                                             & 647                     & 15259.6                                                             & \textbf{57.3}                                                                        & 6424                                                              & \textbf{-41.4}                                                                     & \multicolumn{1}{c}{352}  & \textbf{-45.6}                                                                         \\
                                                                      & 3                         & 4932.1                                                             & 9083                                                              & 612                     & 8912.2                                                             & \textbf{80.7}                                                                        & 8525                                                              & \textbf{-6.2}                                                                      & \multicolumn{1}{c}{469}  & \textbf{-23.4}                                                                         \\ 
                                                                      &  4                        &   911.7                                                           &  53167                                                            &  1546                   &  7189.3                                                            &  \textbf{688.5}                                                                       &  50091                                                            &  \textbf{-5.8}                                                                    & \multicolumn{1}{c}{1416} & \textbf{-8.4}                 \\ 
                                         &  5                        &  5488.1                                                            & 17405                                                              & 918                    & 11016.5                                                             &  \textbf{100.7}                                                                       &   10789                                                           & \textbf{-38.0}                                                                     & \multicolumn{1}{c}{542} & \textbf{-40.9}                 \\\hline
                                         & Average                         &  8806.1                                                            & 19895                                                             & 864                    &  19498.2                                                            & \textbf{213.4}                                                                       & 16371                                                             & \textbf{-24.7}                                                                     &  \multicolumn{1}{c}{622} & \textbf{-32.6}                 \\\hline
\multirow{5}{*}{25x25}                                                & 1                         & 13276.5                                                             & 53347                                                             & 634                     & 19693.2                                                             & \textbf{48.3}                                                                        & 38001                                                             & \textbf{-28.7}                                                                     & \multicolumn{1}{c}{544}  & \textbf{-14.2}                                                                         \\
                                                                      & 2                         & 4475.4                                                             & 51572                                                             & 631                     & 13952.6                                                             & \textbf{211.7}                                                                         & 11799                                                             & \textbf{-77.1}                                                                     & \multicolumn{1}{c}{276}  & \textbf{-56.3}                                                                         \\
                                                                      & 3                         & 1356.3                                                             & 58195                                                             & 655                     & 9635.9                                                             & \textbf{610.4}                                                                        & 10509                                                             & \textbf{-81.9}                                                                     & \multicolumn{1}{c}{258}  & \textbf{-60.6}                                                                         \\ 
                                                                      &  4                        &  12324.7                                                            & 52751                                                             &  624                   &  15097.2                                                            & \textbf{22.5}                                                                        &  9794                                                            &   \textbf{-81.4}                                                                   &  \multicolumn{1}{c}{232} & \textbf{-62.8}                 \\ 
                                         &  5                        &  9546.1                                                            & 16227                                                             &  355                   & 10828.3                                                              &   \textbf{13.4}                                                                      & 23543                                                             &   45.1                                                                   & \multicolumn{1}{c}{415} &  16.9                \\\hline
                                         & Average                         &  8195.8                                                            & 46418                                                             & 580                    &  13841.4                                                            & \textbf{181.3}                                                                       & 18729                                                             & \textbf{-44.8}                                                                     &  \multicolumn{1}{c}{345} & \textbf{-35.4}                 \\\hline
\multirow{5}{*}{30x30}
& 1                        &  4727.5                                                            & 55139                                                             &  717                     & 13024.1                                                               &  \textbf{175.5}                                                                       &  16689                                                            &  \textbf{-69.7}                                                                    & \multicolumn{1}{c}{332} &  \textbf{-53.7}                \\
&  2                        & 21046.2                                                             & 55561                                                             & 720                     & 64567.3                                                               &  \textbf{206.8}                                                                       & 19087                                                             &  \textbf{-65.6}                                                                   & \multicolumn{1}{c}{362} &  \textbf{-49.7}                \\
&  3                        & 3285.4                                                             & 48838                                                             & 664                    &  19604.1                                                           & \textbf{496.7}                                                                        & 13289                                                             &  \textbf{-72.8}                                                                    & \multicolumn{1}{c}{283} &  \textbf{-57.4}                \\
&  4                        & 1751.4                                                             & 58465                                                             &  722                   &  17785.2                                                            &  \textbf{915.5}                                                                       &  58252                                                            & \textbf{-0.36}                                                                     & \multicolumn{1}{c}{692} & \textbf{-4.2}                 \\
&  5                        & 5286.5                                                              & 39259                                                          & 583                    & 19145.3                                                             &  \textbf{262.2}                                                                       & 21992                                                             &  \textbf{-44.0}                                                                    & \multicolumn{1}{c}{396} & \textbf{-32.1}                 \\ \hline
& Average                         &  7219.4                                                            & 51452                                                             & 681                    &  26825.2                                                            & \textbf{411.3}                                                                       & 25862                                                            & \textbf{-50.5}                                                                     &  \multicolumn{1}{c}{413} & \textbf{-39.4}                 \\\hline
\end{tabular}%
\end{adjustbox}
\caption{DKDDPG vs DDPG comparison for large size experiments during training}
\label{tab:largetrain}
\end{sidewaystable}

\begin{figure}[!t]
\begin{minipage}[h]{1.0\linewidth}
\centering
\includegraphics[width=\linewidth]{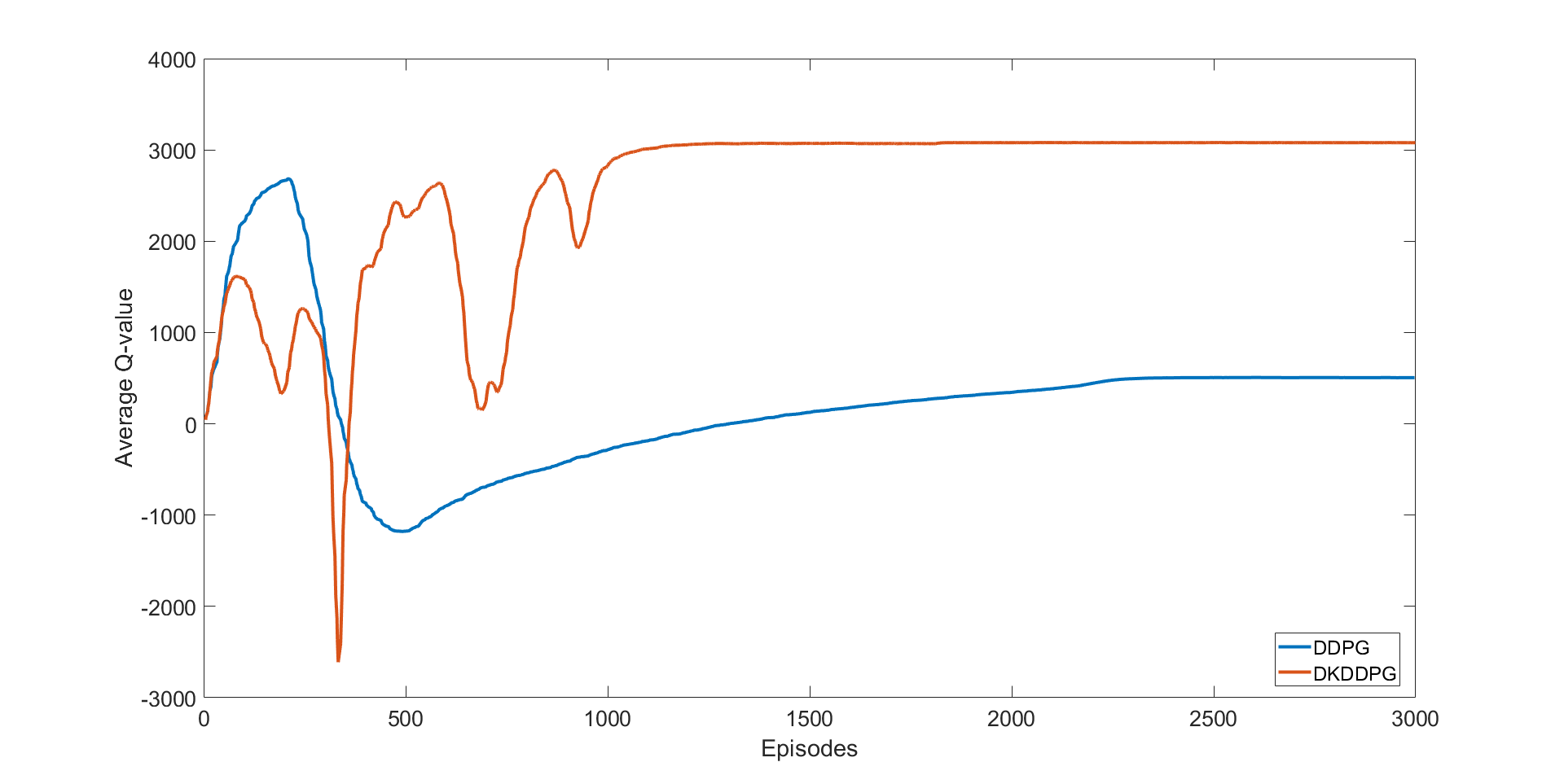}
\subcaption{10x10 Matching Problem}
\end{minipage}

\vspace{0.00mm} 

\begin{minipage}[h]{1.0\linewidth}
\centering
\includegraphics[width=\linewidth]{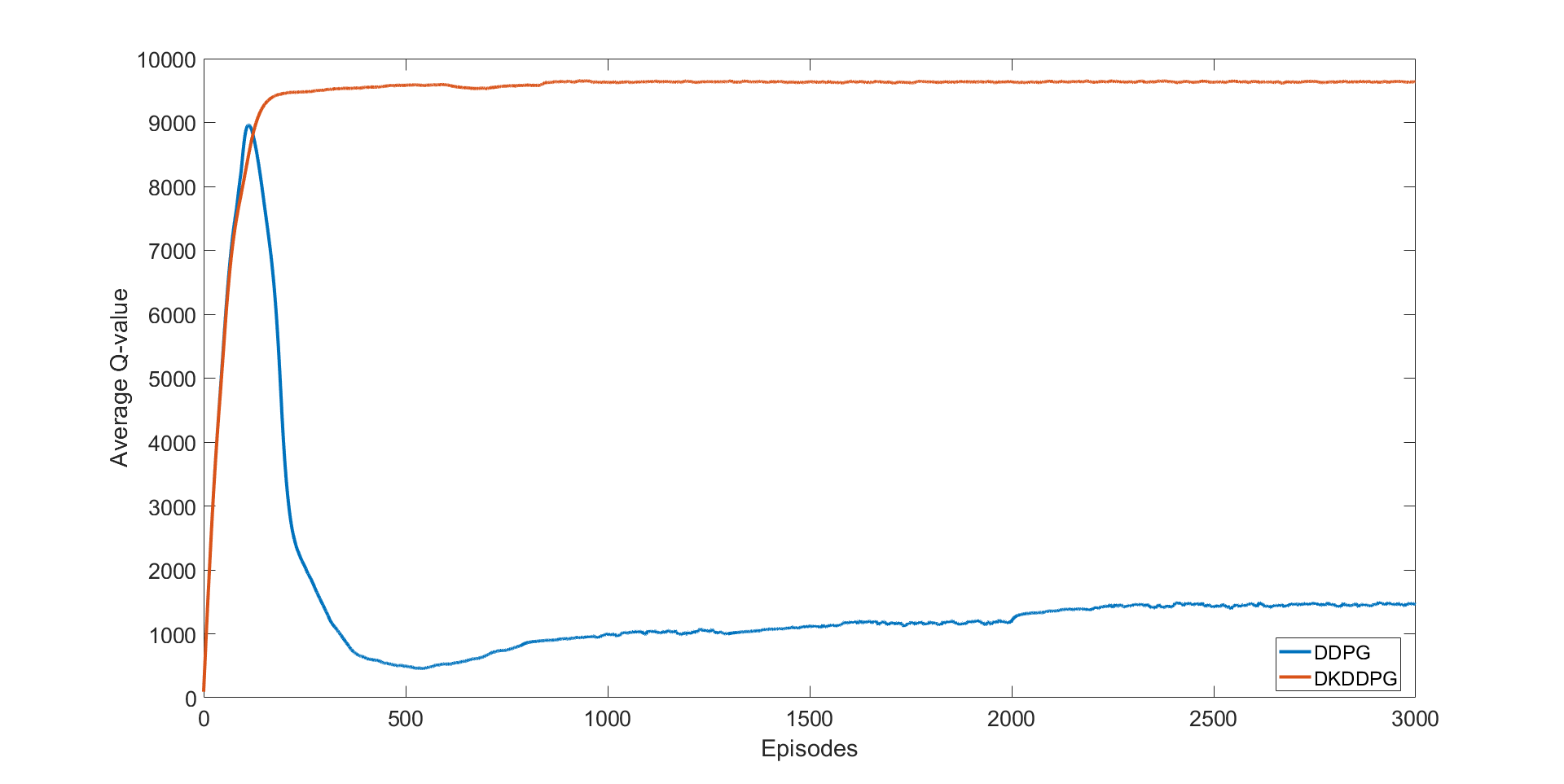}
\subcaption{25x25 Matching problem}
\end{minipage}   


\caption{Comparison of different RL approaches for large-size experiments}
\label{fig:large-case}
\end{figure}


From Table~\ref{tab:largetrain}, we observe that the domain knowledge-informed DDPG achieves convergence approximately 30\% faster in time than DDPG in all the cases. Moreover, DKDDPG takes approximately 36\% lesser number of episodes to solve the matching problem. It is also interesting to note that as the size of the cases increases, the improvement in convergence time and episodes also increases, as observed from the average increment values in time and episode from Table~\ref{tab:largetrain}. This suggests that for larger problems with a high number of demand and capacity types, DKDDPG is much more efficient as compared to DDPG. This validates our hypothesis of adding a prior policy-based penalty to the Q-learning update rule. Upon observing the average learned Q-values, DKDDPG is able to obtain approximately 345\% higher average Q-values as compared to the standard DDPG for all the cases. We note that in some cases, while the DDPG is able to converge sooner than DKDDPG in time, it does so locally and achieves much lower reward values. The accelerated convergence in both time and episode count, along with the increased discounted rewards, indicate that DKDDPG effectively leverages the prior policy-based penalty to mitigate biased estimates and expedite learning. This results in quicker convergence and consistently higher Q-values across all scenarios compared to standard DDPG. Furthermore, in Figure \ref{fig:large-case}, we illustrate the training process for two large-scale experiments, highlighting DKDDPG's superior performance over DDPG in our dynamic matching problem. In both cases presented, DKDDPG consistently achieves significantly higher average Q-values compared to DDPG, which tends to converge to lower values within our 3000 episode limit.

\subsubsection{Out-of-sample performance of DKDDPG}
For the evaluation phase, we test our trained models of DKDDPG and DDPG over a single run of the resource matching problem. We run our trained model while simulating the demand distribution over 500 timesteps. Here, we run the models over five instances with different demand distributions for all the demand-capacity cases, that is, $m=10, 15, 20, 25$ and 30. These instances are separately generated and are thus, different from the ones generated in the training phase. We record the cumulative reward obtained by both the trained models as well as the percentage change of the cumulative reward of DKDDPG over DDPG.
\begin{table}[htb]
\centering
\scalebox{0.6}{
\begin{tabular}{|c|c|c|cc|}
\hline
\multirow{2}{*}{\begin{tabular}[c]{@{}c@{}}Case\\ (mxn)\end{tabular}} & \multirow{2}{*}{Instance} & DDPG                                                               & \multicolumn{2}{c|}{DKDDPG}                                                                                                                      \\ \cline{3-5} 
                                                                        &                           & \begin{tabular}[c]{@{}c@{}}Cumulative Reward\end{tabular} & \begin{tabular}[c]{@{}c@{}}Cumulative Reward\end{tabular} & \begin{tabular}[c]{@{}c@{}}Reward\\ Increment/Decrement\\ (\%)\end{tabular} \\ \hline
\multirow{5}{*}{10x10}                                                  & 1                         & 57640                                                              & 98790                                                              & 71.3                                                                        \\
                                                                        & 2                         & 17700                                                              & 94229                                                              & 432.3                                                                       \\
                                                                        & 3                         & 20097                                                              & 35020                                                              & 74.2                                                                        \\
                                                                        & 4                         & 16136                                                              & 119297                                                             & 639.3                                                                       \\
                                                                        & 5                         & 65906                                                              & 122830                                                             & 86.4                                                                        \\ \hline
                                                            & Average                         &  35496                                                            & 94033                                                             & 260.7\\\hline
\multirow{5}{*}{15x15}                                                  & 1                         & 27545                                                              & 47625                                                              & 72.9                                                                        \\
                                                                        & 2                         & 37118                                                              & 68237                                                              & 83.8                                                                        \\
                                                                        & 3                         & 28546                                                              & 78356                                                              & 174.5                                                                       \\
                                                                        & 4                         & 13447                                                              & 56321                                                              & 318.8                                                                       \\
                                                                        & 5                         & 53487                                                              & 87612                                                              & 63.8                                                                        \\ \hline& Average                         &  32029                                                            & 67630                                                             & 142.8\\\hline
\multirow{5}{*}{20x20}                                                  & 1                         & 28994                                                              & 39554                                                              & 36.4                                                                        \\
                                                                        & 2                         & 36670                                                              & 54933                                                              & 49.8                                                                        \\
                                                                        & 3                         & 30812                                                              & 43090                                                              & 39.9                                                                        \\
                                                                        & 4                         & 47200                                                              & 67478                                                              & 42.9                                                                        \\
                                                                        & 5                         & 41698                                                              & 48933                                                              & 17.3                                                                        \\ \hline& Average                         &  37075                                                            & 50798                                                             & 37.3\\\hline
\multirow{5}{*}{25x25}                                                  & 1                         & 379956                                                             & 527663                                                             & 38.9                                                                        \\
                                                                        & 2                         & 843976                                                             & 976238                                                             & 15.7                                                                        \\
                                                                        & 3                         & 538745                                                             & 679832                                                             & 26.2                                                                        \\
                                                                        & 4                         & 306357                                                             & 475368                                                             & 55.2                                                                        \\
                                                                        & 5                         & 481672                                                             & 538135                                                             & 11.7                                                                        \\ \hline& Average                         &  510141                                                            & 639447                                                             & 29.5\\\hline
\multirow{5}{*}{30x30}                                                  & 1                         & 685083                                                             & 831725                                                             & 21.4                                                                        \\
                                                                        & 2                         & 721876                                                             & 904137                                                             & 25.2                                                                        \\
                                                                        & 3                         & 377772                                                             & 542611                                                             & 43.6                                                                        \\
                                                                        & 4                         & 432175                                                             & 852467                                                             & 97.3                                                                        \\
                                                                        & 5                         & 483467                                                             & 655734                                                             & 35.6                                                                        \\ \hline& Average                         &  540075                                                            & 747335                                                             & 44.6\\\hline
\end{tabular}}
\caption{DKDDPG vs DDPG comparison for large size experiments during evaluation}
\label{tab:largeeval}
\end{table}

From Table~\ref{tab:largeeval}, we observe that the trained model of DKDDPG is able to obtain higher cumulative reward over 500 timesteps as compared to the DDPG model. Specifically, the DKDDPG model achieves an average reward increment of 103\% over DDPG amongst the 5 supply-demand cases.
\section{Conclusion \& Future work}
In this paper, we considered the problem of dynamically matching the demand-capacity types of resources in manufacturing. We formulated the problem as an MDP where the outstanding demand at the end of a period is considered the state and the quantity of demand and capacity matched is the action matrix. To tackle the issue of biased estimates and infeasible actions, we introduced prior policy-based penalty and infeasibility penalty respectively into the traditional Q-learning algorithm. We considered two cases of the regularization parameter $\beta$: constant and scheduling, and further theoretically proved the convergence of our domain knowledge-informed Q-learning algorithm for both $\beta$ cases. To avoid the curse of dimensionality, we proposed the DKDDPG algorithm which utilizes our modified Q-learning update rule. We investigated the performance of our DKDDPG algorithm with some benchmark RL algorithms for both small and large-size experiments and were able to demonstrate improvement in performance and efficiency over those algorithms.

In this paper, we considered a dynamic resource-matching problem in manufacturing for a matching firm with centralized control. It is worth considering a decentralized, cooperative problem with multiple agents having only local information. While our problem considers a single-period lead time, it will be interesting to model a system with a multi-period lead time to accommodate for delays in production due to numerous reasons. We have not considered any matching costs in our system; including a supply-based cost in our reward framework could aid in minimizing the number of suppliers to fulfill a given demand type. We provided theoretical results for domain knowledge-informed Q-learning which establishes a performance guarantee for small-size problems. Future work will consider deriving convergence results for algorithms with function approximations.
\section*{Acknowledgement}
This work is supported in part by the U.S. National Science Foundation under award 2305486.
\appendix
\section{Proof of Theorem~\ref{thm:THMONE}}\label{sec:appendixA}
We shall first state the result by \citet{bertsekas1996neuro} provided as Lemma~\ref{lem:LM1} below.
\begin{lemma}\label{lem:LM1}
Let $f_t$ be the sequence generated by the iteration (15). Given the conditions in Assumption 1 are satisfied, then $f_t$ converges to $f^*$ with probability 1. 
\end{lemma}
Notice that at the optimal policy, the optimal domain knowledge-informed Q-function satisfies
\begin{align*}
    Q^*(s,a) &= r + \gamma \max_\pi[F^\pi(s')|s,a]\\
             &\equiv \textbf{B}^*[Q^*]_{(s,a)}
\end{align*}

The contraction property of operator $\textbf{B}^*$ defined above can be proven similarly as done by \citet{fox2015taming}.
\thmone*
\begin{proof}
We can see that the 1st assumption above is the same as condition (a) stated in Assumption~\ref{asm:ASMONE}. Since $\textbf{B}^*$ is a contraction mapping, $\textbf{B}^*$ is automatically a pseudo-contraction, satisfying condition (d) of Assumption~\ref{asm:ASMONE}.

We now verify the assumptions on the noise variable $\omega_t$. Using the definition of $\textbf{B}^*$, we automatically get \begin{align*}
    \mathrm{E}[\omega_t(r_t,s_{t+1})|\mathcal{F}_t]&=\mathrm{E}\big[-\textbf{B}^*[Q_t]_{(s_t,a_t)} + r_t + \gamma \mathrm{E}_p[F^\pi(s_{t+1})|s_t,a_t]\big|\mathcal{F}_t\big]\\
                                                   &= -\textbf{B}^*[Q_t]_{(s_t,a_t)} + \mathrm{E}\big[r_t + \gamma \mathrm{E}_p[F^\pi(s_{t+1})|s_t,a_t]\big|\mathcal{F}_t\big]\\
                                                   &= 0
\end{align*}
Hence, the condition (b) of Assumption~\ref{asm:ASMONE} is satisfied. Now, we shall verify condition (c).
\begin{align*}
    \mathrm{E}[\omega_t^2(r_t,s_{t+1})|\mathcal{F}_t] &= \mathrm{E}\Big[\big(-\textbf{B}^*[Q_t]_{(s_t,a_t)} + r_t + \gamma \mathrm{E}_p[F^\pi(s_{t+1})|s_t,a_t]\big)^2\Big|\mathcal{F}_t\Big]\\
    &= \mathrm{E}\Big[\big(-\textbf{B}^*[Q_t]_{(s_t,a_t)}\big)^2 + \big(r_t + \gamma \mathrm{E}_p[F^\pi(s_{t+1})|s_t,a_t]\big)^2\\ &\qquad - 2\textbf{B}^*[Q_t]_{(s_t,a_t)}\big(r_t + \gamma \mathrm{E}_p[F^\pi(s_{t+1})|s_t,a_t]\big)\Big|\mathcal{F}_t\Big]
    \intertext{Taking expectation for each term, we get}
    &=\mathrm{E}\Big[\big(-\textbf{B}^*[Q_t]_{(s_t,a_t)}\big)^2 \Big|\mathcal{F}_t\Big]
    + \mathrm{E}\Big[\big(r_t + \gamma \mathrm{E}_p[F^\pi(s_{t+1})|s_t,a_t]\big)^2 \Big|\mathcal{F}_t\Big]
    \\ &\qquad - 2
    \mathrm{E}\Big[\big(\textbf{B}^*[Q_t]_{(s_t,a_t)}\big)\big(r_t + \gamma \mathrm{E}_p[F^\pi(s_{t+1})|s_t,a_t]\big)\Big|\mathcal{F}_t\Big]\\
    &=\big(-\textbf{B}^*[Q_t]_{(s_t,a_t)}\big)^2
    + \mathrm{E}\Big[\big(r_t + \gamma \mathrm{E}_p[F^\pi(s_{t+1})|s_t,a_t]\big)^2 \Big|\mathcal{F}_t\Big]
    \\ &\qquad - 2
    \mathrm{E}\Big[\big(\textbf{B}^*[Q_t]_{(s_t,a_t)}\big)\Big|\mathcal{F}_t\Big]\mathrm{E}\Big[\big(r_t + \gamma \mathrm{E}_p[F^\pi(s_{t+1})|s_t,a_t]\big)\Big|\mathcal{F}_t\Big]\\
    &= \big(\textbf{B}^*[Q_t]_{(s_t,a_t)}\big)^2
     + \mathrm{E}\Big[\big(r_t + \gamma \mathrm{E}_p[F^\pi(s_{t+1})|s_t,a_t]\big)^2 \Big|\mathcal{F}_t\Big]
    - 2 \big(\textbf{B}^*[Q_t]_{(s_t,a_t)}\big)\big(\textbf{B}^*[Q_t]_{(s_t,a_t)}\big)\\
    &= \mathrm{E}\Big[\big(r_t + \gamma \mathrm{E}_p[F^\pi(s_{t+1})|s_t,a_t]\big)^2 \Big|\mathcal{F}_t\Big] - \big(\textbf{B}^*[Q_t]_{(s_t,a_t)}\big)^2\\
    &\leq A + B\big|\big|Q_t\big|\big|^2
\end{align*}
where $A\geq \mathrm{E}\Big[\big(r_t + \gamma \mathrm{E}_p[F^\pi(s_{t+1})|s_t,a_t]\big)^2 \Big|\mathcal{F}_t\Big] $ and $B=-1$ are constants. Since, all the conditions of Assumption~\ref{asm:ASMONE} are satisfied, hence using Lemma~\ref{lem:LM1} we conclude that $Q_t$ converges to $Q^*$ with probability 1. 
\end{proof}
\setcounter{equation}{0}
\renewcommand{\theequation}{\thesection.\arabic{equation}}
\section{Proof of Theorem~\ref{thm:THMTWO}}\label{sec:appendixB}
We use a result provided by \citet{singh2000sarsa} denoted as Lemma~\ref{lem:LM2} to help prove Theorem \ref{thm:THMTWO}.
\begin{lemma}\label{lem:LM2}
Consider a stochastic process $(\alpha_t, \Delta_t, H_t), t\geq0$, where $\alpha_t,\Delta_t, H_t:X \longrightarrow \mathcal{R}$ satisfy the equations
\begin{align*}
    \Delta_{t+1}(x) = (1-\alpha_t(x))\Delta_t(x) + \alpha_t(x)H_t(x), \qquad x \in X, \qquad t=0,1,2,\dotso
\end{align*}
Let $P_t$ be a sequence of increasing $\sigma-$fields such that the $\alpha_0$ and $\Delta_0$ are $P_0-$measurable and $\alpha_t,\Delta_t$ and $H_{t-1}$ are $P_t-$measurable, $t=1,2,\dotso$. Assume that the following hold:
\begin{enumerate}
    \item the set X is finite.
    \item $0\leq \alpha_t(x) \leq 1, \sum_t \alpha_t(x)=\infty, \sum_t \alpha_t^2(x) < \infty$ w.p.1.
    \item $||E\{H_t(\cdot)|P_t\}||_W \leq \kappa||\Delta_t||_W + c_t$, where $\kappa \in [0,1)$ and $c_t$ converges to zero w.p.1.
    \item $Var\{H_t(X)|P_t\} \leq K(1 + ||\Delta_t||_W)^2$, where $K$ is some constant.
\end{enumerate}
Then, $\Delta_t$ converges to zero with probability one(w.p.1).
\end{lemma}
\thmtwo*
\begin{proof}
Subtracting the optimal $Q^*$ on both sides of equation (9), we get
\begin{align*}
    \Delta_{t+1}(s_t,a_t) &= (1-\alpha_t)\Delta_t(s_t,a_t) + \alpha_t(r_t + \gamma \max_\pi F^\pi_t(s_{t+1}) - Q^*(s_t,a_t))
    \intertext{where, $\Delta_t(s_t,a_t) = Q_t(s_t,a_t)-Q^*(s_t,a_t)$. We shall denote $H_t(s_t,a_t) = r_t + \gamma \max_\pi F^\pi_t(s_{t+1}) - Q^*(s_t,a_t)$. Thus, rewriting it, we get}
    \Delta_{t+1}(s_t,a_t) &= (1-\alpha_t)\Delta_t(s_t,a_t) + \alpha_t H_t(s_t,a_t)
\end{align*}

Notice that the optimal Q-function for a given $\beta_t$ satisfies
\begin{align}
    Q^*_t(s_t,a_t) &= r + \gamma\max_\pi F^\pi_t(s_{t+1}) \label{eq:defBstarT1}\\
               &= r + \gamma\max_\pi \sum_{a_{t+1}} \pi(a_{t+1}|s_{t+1}) \bigg[\frac{1}{\beta_t} g^\pi(s_{t+1},a_{t+1}) + Q^\pi_t(s_{t+1},a_{t+1}) \bigg]\\
             &\equiv \textbf{B}^*_t[Q_t]_{(s_t,a_t)} \label{eq:defBstarT3}
\end{align}

Now, as the value of $\beta_t$ increases to a large number, $\frac{1}{\beta_t} g^\pi(s_{t+1},a_{t+1}) \longrightarrow 0$. So, the optimal Q-function as $\beta_t \longrightarrow \infty$ is given by
\begin{align}
    Q^*(s_t,a_t) 
             &= r + \gamma \max_\pi \sum_{a'} \pi(a_{t+1}|s_{t+1})Q^\pi_t(s_{t+1},a_{t+1}) \label{eq:defBstar1}\\
             &\equiv \textbf{B}^*[Q_t]_{(s_t,a_t)} \label{eq:defBstar2}
\end{align}
We verify all the assumptions in Lemma~\ref{lem:LM2} to prove the convergence of domain knowledge-inspired Q-learning with scheduled $\beta_t$. Assumptions 1 and 2 of Lemma~\ref{lem:LM2} are trivially satisfied for our algorithm. We need to check assumptions 3 and 4 in Lemma~\ref{lem:LM2}. Let $P_t = \{Q_{0}(s,a),\dotso,Q_{t}(s,a),H_{0}(s,a),\\\dotso,H_{t}(s,a)\}$, for $s \in \mathcal{S}, a \in \mathcal{A}$. Then,
\begin{align*}
    \big|\big|\mathrm{E}\{H_t(s_t,a_t)|P_t\}\big|\big|_W &= \big|\big|\mathrm{E}\{r + \gamma \max_\pi F^\pi_t(s_{t+1}) - Q^*(s_t,a_t)|P_t\}\big|\big|_W\\
       &= \big|\big|\mathrm{E}\{r + \gamma \max_\pi F^\pi_t(s_{t+1})|P_t\} - \mathrm{E}\{Q^*(s_t,a_t)|P_t\}\big|\big|_W
       \intertext{Using Equation~\ref{eq:defBstarT1} and \ref{eq:defBstarT3}}
       &= \big|\big|\textbf{B}^*_t[Q_t]_{(s_t,a_t)} - Q^*(s_t,a_t)\big|\big|_W
       \intertext{Adding and subtracting $Q^*_t$,}
       &=\big|\big|\textbf{B}^*_t[Q_t]_{(s_t,a_t)} - Q^*_t(s_t,a_t) + Q^*_t(s_t,a_t) - Q^*(s_t,a_t)\big|\big|_W\\
       &\leq \big|\big|\textbf{B}^*_t[Q_t]_{(s_t,a_t)} - Q^*_t(s_t,a_t)\big|\big|_W + \big|\big|Q^*_t(s_t,a_t) - Q^*(s_t,a_t)\big|\big|_W
       \intertext{Applying the definitions from Equation~\ref{eq:defBstar2} to the first term,}
       &= \big|\big|\textbf{B}^*_t[Q_t]_{(s_t,a_t)} - \textbf{B}^*_t [Q_t]_{(s_t,a_t)}\big|\big|_W + \big|\big|Q^*_t(s_t,a_t) - Q^*(s_t,a_t)\big|\big|_W\\
       &= \big|\big|Q^*_t(s_t,a_t) - Q^*(s_t,a_t)\big|\big|_W
       \intertext{By definitions \ref{eq:defBstarT3} and \ref{eq:defBstar2},}
       |Q^*_t(s_t,a_t)-Q^*(s_t,a_t)| &= |\textbf{B}^*_t[Q_t]_{(s_t,a_t)}-\textbf{B}^*[Q_t]_{(s_t,a_t)}|\\
       &= \Big|\gamma \max_\pi F^\pi_t(s_{t+1})-\gamma \max_\pi \sum_{a_{t+1}}\pi(a_{t+1}|s_{t+1})Q^\pi_t(s_{t+1},a_{t+1})\Big|\\
       &\leq \gamma \max_\pi \Big|F^\pi_t(s_{t+1})-\sum_{a_{t+1}}\pi(a_{t+1}|s_{t+1})Q^\pi_t(s_{t+1},a_{t+1})\Big|\\
       &\leq \gamma \max_\pi \sum_{a_{t+1}}\pi(a_{t+1}|s_{t+1})\Big|\frac{1}{\beta_t}g^\pi(s_{t+1},a_{t+1})+Q^\pi_t(s_{t+1},a_{t+1})-Q^\pi_t(s_{t+1},a_{t+1})\Big|\\
        &=\gamma \max_\pi \sum_{a_{t+1}}\pi(a_{t+1}|s_{t+1})\Big|\frac{1}{\beta_t}g^\pi(s_{t+1},a_{t+1})\Big|
       \end{align*}
Assume $g^\pi(s_{t+1},a_{t+1})$ is bounded by a constant $C$ for all $\pi,s_{t+1},a_{t+1}$. It follows that
    \begin{align*}
                        |Q^*_t(s_t,a_t)-Q^*(s_t,a_t)| &\leq \frac{\gamma}{\beta_t}C
                        \intertext{Therefore, we have}
                        \big|\big|\Delta^*_t(s_t,a_t)\big|\big|_W &= \big|\big|Q^*_t(s_t,a_t) - Q^*(s_t,a_t)\big|\big|_W\\ &\leq \frac{\gamma}{\beta_t}C \longrightarrow 0.
                        \intertext{Thus,}
                        \big|\big|\mathrm{E}\{H_t(s_t,a_t)|P_t\}\big|\big|_W &\leq \kappa \big|\big| \Delta_t(s_t,a_t) \big|\big|_W + c_t
\end{align*}
where $c_t=\gamma\big|\big| - \Delta^*_t(s_t,a_t) \big|\big|_W + \big|\big| \Delta^*_t(s_t,a_t) \big|\big|_W$ which converges to zero w.p.1 and $\kappa=0$ in the 1st term. Thus, assumption 3 is verified.
We shall now verify assumption 4 in Lemma~\ref{lem:LM2}.
\begin{align*}
    Var\{H_t(s_t,a_t)|P_t\} &= \mathrm{E}\{H_t^2(s_t,a_t)|P_t\} - \mathrm{E}\{H_t(s_t,a_t)|P_t\}^2\\
                            &= \mathrm{E}\big\{\big(r + \gamma \max_\pi F^\pi_t(s_{t+1}) - Q^*(s_t,a_t)\big)^2|P_t\big\}\\
                            &\qquad \qquad - \mathrm{E}\big\{r + \gamma \max_\pi F^\pi_t(s_{t+1}) - Q^*(s_t,a_t)|P_t\big\}^2\\
                            &= \mathrm{E}\big\{\big(r+\gamma \max_\pi F^\pi_t(s_{t+1})\big)^2 + \big(Q^*(s_t,a_t)\big)^2 - 2\big(r+\gamma \max_\pi F^\pi_t(s_{t+1})\big)Q^*(s_t,a_t)|P_t\big\}\\
                            &\qquad \qquad - \mathrm{E}\big\{r + \gamma \max_\pi F^\pi_t(s_{t+1}) - Q^*(s_t,a_t)|P_t\big\}^2\\
                            &= \mathrm{E}\big\{\big(r+\gamma \max_\pi F^\pi_t(s_{t+1})\big)^2|P_t\big\} + \mathrm{E}\big\{\big(Q^*(s_t,a_t)\big)^2|P_t\big\} \\
                            &\qquad \qquad -2 \mathrm{E}\big\{\big(r+\gamma \max_\pi F^\pi_t(s_{t+1})\big)Q^*(s_t,a_t)|P_t\big\} - \big(\mathrm{E}\big\{\big(r + \gamma \max_\pi F^\pi_t(s_{t+1})\big)|P_t\big\}\\
                            &\qquad \qquad \qquad -\mathrm{E}\big\{Q^*(s_t,a_t)|P_t\big\}\big)^2
                            \intertext{Applying the definitions \ref{eq:defBstarT3} and \ref{eq:defBstar2},}
                            &= \big(\textbf{B}^*_t[Q_t]_{(s_t,a_t)}\big)^2 + \big(\textbf{B}^*[Q_t]_{(s_t,a_t)}\big)^2 - 2\textbf{B}^*_t[Q_t]_{(s_t,a_t)}\textbf{B}^*[Q_t]_{(s_t,a_t)}\\
                            &\qquad \qquad- \big(\textbf{B}^*_t[Q_t]_{(s_t,a_t)}-\textbf{B}^*[Q_t]_{(s_t,a_t)}\big)^2
                            \intertext{Combining the first 3 terms above,}
                            &= \big(\textbf{B}^*_t[Q_t]_{(s_t,a_t)}-\textbf{B}^*[Q_t]_{(s_t,a_t)}\big)^2 - \big(\textbf{B}^*_t[Q_t]_{(s_t,a_t)}-\textbf{B}^*[Q_t]_{(s_t,a_t)}\big)^2\\
                            &= K(1 + ||\Delta_t||_W)^2
\end{align*}
where $K=0$, hence making it zero variance. Thus, assumption 4 is verified. Since all the assumptions are verified, $Q_t \longrightarrow Q^*$ w.p.1 by Lemma~\ref{lem:LM2}.
\end{proof}
\bibliography{Refpapers}
\end{document}